\documentclass[sn-mathphys,Numbered]{sn-jnl}

\usepackage{graphicx}%
\usepackage{multirow}%
\usepackage{amsmath,amssymb,amsfonts}%
\usepackage{amsthm}%
\usepackage{mathrsfs}%
\usepackage[title]{appendix}%
\usepackage{xcolor}%
\usepackage{textcomp}%
\usepackage{manyfoot}%
\usepackage{booktabs}%
\usepackage{listings}%

\usepackage{figlatex}  

\usepackage{enumitem}
\usepackage[utf8]{inputenc} 
\usepackage[T1]{fontenc}    
\usepackage{hyperref}       
\usepackage{url}            
\usepackage{amsfonts}       
\usepackage{nicefrac}       
\usepackage{microtype}      
\usepackage{xcolor}         

\usepackage{textcomp} 
\usepackage[english]{babel} 
\usepackage[autolanguage]{numprint} 
\usepackage{hyperref} 
\usepackage{verbatim} 
\usepackage{amsmath,amsthm} 
\usepackage{dsfont}
\usepackage[linesnumbered,ruled]{algorithm2e} 
\usepackage{xspace}

\usepackage{tikz}
\usepackage{caption}
\usepackage{float}

\theoremstyle{definition} 
\newtheorem{Def}{Definition}[section] 
\theoremstyle{plain} 
\newtheorem{Pro}[Def]{Proposition} 
\newtheorem{Lem}[Def]{Lemma} 
\newtheorem{The}[Def]{Theorem} 
\newtheorem{Cor}[Def]{Corollary} 
\theoremstyle{remark} 

\newcommand{\N}{\mathbb{N}}

\newcommand{\E}{{\mathbb{E}}}
\newcommand{\Pb}{{\mathbb{P}}}
\newcommand{\1}{{\mathds{1}}}
\newcommand{\prt}[1]{\left(#1\right)}
\newcommand{\brac}[1]{\left[#1\right]}
\newcommand{\croc}[1]{\left\{#1\right\}}

\newcommand{\bs}{{s}}

\newcommand{\Reg}{\mbox{Reg}}
\newcommand{\bx}{\mathbf{x}}
\newcommand{\tm}{\tau_{\rm mix}}
\newcommand{\mrate}{\rho} 
\newcommand{\sm}{\nu} 
\newcommand{\gain}{g}
\newcommand{\rcost}{\gamma_{{\rm reject}}} 
\newcommand{\hcost}{\gamma_{{\rm hold}}} 
\newcommand{\module}{m} 

\newtheorem{theorem}{Theorem}
\newtheorem{proposition}[theorem]{Proposition}%

\begin{document}

\title{Learning Optimal Admission Control in Partially Observable Queueing Networks}


\author[1]{\fnm{Jonatha} \sur{Anselmi}}\email{jonatha.anselmi@inria.fr}
\equalcont{These authors contributed equally to this work.}

\author[1]{\fnm{Bruno} \sur{Gaujal}}\email{bruno.gaujal@inria.fr}
\equalcont{These authors contributed equally to this work.}

\author*[1]{\fnm{Louis-S\'ebastien} \sur{Rebuffi}}\email{louis-sebastien.rebuffi@univ-grenoble-alpes.fr}

\affil*[1]{\orgname{Univ.\ Grenoble Alpes, Inria, CNRS, Grenoble INP, LIG,} \city{Grenoble}, \postcode{38000}, \country{France}}




\abstract{We present an efficient  reinforcement learning algorithm that learns the optimal admission control policy in a partially observable queueing network.
Specifically, only the arrival and departure times from the network are observable, and optimality  refers to the average holding/rejection cost in infinite horizon.

While reinforcement learning in Partially Observable Markov Decision Processes (POMDP) is prohibitively expensive in general, we show that our algorithm has a regret that only depends \emph{sub-linearly} on the maximal number of jobs in the network, $S$. In particular, in contrast with existing regret analyses, our regret bound does not depend on the \emph{diameter} of the underlying Markov Decision Process~(MDP), which in most  queueing systems is at least exponential in~$S$.

The novelty of our approach is to leverage Norton's equivalent theorem for closed product-form queueing networks and an efficient reinforcement learning algorithm for  MDPs with the structure of birth-and-death processes.}




\maketitle

\section{Introduction}

Research in reinforcement learning in Markov Decision Processes (MDPs) has been thriving since the  work of Walkins \cite{Walkins_1989} on Q-learning, the celebrated model-free learning algorithm.
%
%
Since then, several extensions of Q-learning~\cite{Jin-2019}, Bayesian~\cite{ouyang-2017,TSDE_2017} or model-based~\cite{jaksch-2010} approaches
have been proposed to improve learning efficiency \cite{Fruit-2018,Tossou-2019} up to the point where the \emph{regret} of the most recent algorithms is ``asymptotically optimal'' in the sense that it matches
a universal lower bound:
The regret of the best learning algorithm with respect to an optimal policy is $\tilde{O}( \sqrt{DSAT} )$\footnote{The $\tilde{O}$ notation is a variant of big-O that ignores logarithmic factors.}, where $S$ is the number of states, $A$ the number of actions, $T$ the time horizon and $D$ the \emph{diameter} of the MDP; we point the reader to~\cite{wei_model-free_2020} for a recent overview of this quest for optimal regret.

Since this problem has reached a satisfactory solution, the following natural question arises:
Can one learn \emph{efficiently} the optimal policy of an MDP not only when the rewards and the transition kernel are unknown \emph{but also when the state is partially observable}?
Recently, this question has been investigated under certain assumptions on the structure of model parameters~\cite{Jin-2020,Lazaric16,GuoDB16}.
In this paper, we address this question
in the context of queueing networks where we assume that the learner has only access to the \emph{total} number of jobs in the network, and this makes our problem fall in the family of Partially Observable MDPs (POMDPs).

\subsection{Reinforcement learning in POMDPs}

It is well-known that POMDPs are prohibitively expensive to solve.
If the parameters are known, the problem of computing an optimal policy is PSPACE-complete even in finite horizon \cite{PapTsi}. Furthermore, it is NP-hard to compute the optimal memoryless policy \cite{Nikos2012}.
In reinforcement learning, where (some of) the model parameters are unknown,
the lower bound on the average-case complexity developed in \cite[Propositions~1 and~2]{Jin-2020} confirms with no surprise that reinforcement learning in POMDPs remains intractable.
Matter of fact, the design of effective exploration--exploitation strategies in POMDPs is still relatively unexplored; see~\cite[Section 1]{Lazaric16} for a detailed discussion.
In the attempt to reduce this computational burden, researchers focused on reinforcement learning in \emph{subclasses} of POMDPs~\cite{Jin-2020}, and we will also follow this approach.
The algorithm in \cite{resets05} assumes POMDPs without resets and has sample complexity scaling exponentially with a certain horizon time.
The Bayesian algorithms proposed in \cite{NIPS2007_3b3dbaf6,Poupart2008ModelbasedBR} learn POMDPs but bounds on the mean regret remain unknown for these approaches.
%
A sample-efficient algorithm for episodic finite POMDPs is given in \cite{Jin-2020}. Here, it is assumed that the number of observations is larger than the number of latent states.

The works above have focused on reinforcement learning over a finite or discounted horizon. In contrast, we will be interested in the (undiscounted) infinite horizon case, which is technically more challenging.
In infinite horizon, a POMDP algorithm based on spectral methods is proposed in \cite{Lazaric16}.
For this algorithm, the authors find an order-optimal regret bound with respect to the optimal memoryless policy.
However, it exhibits a linear dependence on the diameter $D$ of the underlying MDP.
This dependence makes this type of bounds not interesting in the context of queueing systems as the diameter is usually exponential in the number of states~\cite{anselmi-2022}.
Although the additional assumptions on the structure of the model mitigate, to some extent, the intrinsic complexity of POMDPs, learning algorithms with regret $\tilde{O}(\sqrt{DSAT})$ have remained elusive for all but trivial cases to the best of our knowledge.

\subsection{Contribution and methodology}

In this paper, we propose a learning algorithm for the optimal job-admission policy in a partially observable queueing network with regret $\Reg(T) \leq \tilde O (\sqrt{ST})$.
Thus, our main contribution is a learning algorithm with a regret bound that
 does not depend on the diameter $D$,  and  whose dependence on the state space is very small.

Optimal admission control  is one of the most classical control problems in queues. It has been investigated in several works; see, e.g., \cite{Borgs-2014,Xia} and the references therein. However, these works  consider the case where the model parameters are known, i.e., no learning mechanism is used.
The novelty of our approach is to leverage i) Norton's equivalence theorem for closed product-form queueing networks \cite{Norton} and ii) the efficiency of reinforcement learning in MDPs with the structure of birth-and-death processes \cite{anselmi-2022}.
More specifically,
our result is achieved by using Norton's theorem to replace the whole network by a single load-dependent queue in its stationary regime and relies on the mixing time $\tm $ of the network to apply this equivalence every $\tm $ time-steps.
The key observation is that
Norton's theorem helps us to somewhat cast the original partially-observable MDP to a standard (fully-observable) MDP.
In other words, the resulting asymptotically equivalent POMDP becomes an MDP with the structure of a birth and death process. This structure is then exploited to construct tight bounds on the regret of our algorithm by controlling the bias of the current policy as well as its stationary measure.


%

\subsection{Organization}

The remainder of the paper is organized as follows.
The model of the queueing network, its practical motivation and Norton's equivalent queue are presented in Section \ref{sec:network}.
Section~\ref{sec:MDP} presents the problem addressed in the paper in detail, Section~\ref{sec:algo} is dedicated to the presentation of our learning algorithm (UCRL-M) and Section~\ref{sec:regret} to the analysis of its regret.
In the latter, we state our main result in Theorem~\ref{th:main}.
Then, Section~\ref{sec:mixing} discusses some technical aspects of our regret bound. Section~\ref{sec:experiment} showcases the behaviour of the algorithm on a multi-tier queueing network, and,
finally, Section~\ref{sec:conclusion} draws the conclusions of our work.

\section{Admission control in a queueing network}
 \label{sec:network}


We consider a  Jackson network with~$N$ queues (or stations) having service rates $\mu_1^o,\ldots,\mu_N^o$, a routing probability matrix~$L=[L_{i,j}:0\le i,j\le N]$ and exogenous arrivals occurring with rate~$\lambda$.  Here,
$L_{0,j}$ (resp. $L_{i,0}$) represents the probability that a job joins queue~$j$ from outside (resp. leaves the network after service at~ queue $i$).
To guarantee stability, i.e., positive recurrence of the underlying Markov chain, we require that $\lambda_i^o < \mu_i^o$, for all~$i$, where $\lambda_i^o$ is the arrival rate at queue~$i$. It is given by the unique solution of the traffic equations $\lambda_i^o = \lambda L_{0,i} + \sum_{j}\lambda_j^o L_{j,i} $, for all $i$.

We further assume that the total number of jobs in the network cannot exceed~$S$.
Under this constraint, the global system can be  seen as a \emph{closed} network with $S$ jobs. This network is identical to the original one except for an auxiliary queue, say queue 0, that represents the outside world with service rate $\lambda$. The departures of queue 0 correspond to the arrivals of the initial open  network (see Figure~\ref{fig:MDP}).

  \vspace{-0.2cm}
\begin{figure}[hbtp]
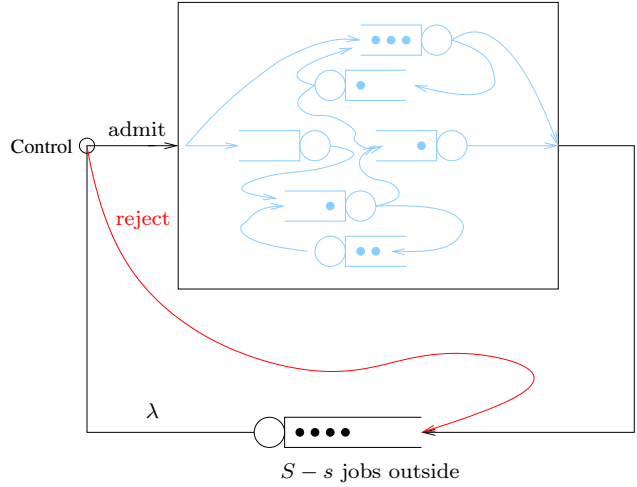

  \centering
  \include{MDPcode.tex}
  \vspace{-0.7cm}
  \caption{Admission control: rejected jobs immediately return to the outside queue.}
  \label{fig:MDP}
\end{figure}

Jobs that want to enter the network are subjected to admission control. If a job is rejected this can be modeled in the closed network  as the job being sent back to the outside queue (see Figure~\ref{fig:MDP}).
The goal of the admission controller is to minimize a cost function.
For each job, the immediate cost $c_t$ is decomposed into a per-rejection cost $\rcost$ and a per-time-unit holding cost $\hcost$.
This cost function is the long-run average cost per time unit: $\lim_{T\to \infty} \frac{1}{T}\sum_{t=0}^T c_t$.

When the controller can observe the state of the network and knows the parameters of the system $(\lambda,\mu^o,L)$,  this classical problem has been solved in \cite{Xia}. 
In the case where the network is a single M/M/1 queue, there exists an explicit formula (involving the Lambert $W$ function) for the optimal admission policy~\cite{Borgs-2014}.

 \subsection{Problem formulation}
 \label{sec:MDP}

%
%

In this paper, we consider an admission controller that can only observe arrivals to and departures from the network. More precisely, the network topology, internal service rates and routing probabilities are not known, and the movements of the jobs inside the network are not observable.
Our objective is to design a learning algorithm that learns the optimal admission policy with a \emph{small} regret in the sense that its dependence on the network \emph{complexity} is minimal.

For the cost to be minimized, we assume that:
\begin{itemize}
\item the controller may choose to reject jobs arriving in the network, at the price of a fixed $\rcost$ for each rejected job;
\item for every time unit in the system, each job induces a holding cost $\hcost$ (this is the classical cost function for admission control, see \cite{Borgs-2014});
\item the controller takes decisions only relying on its set of observations up to time $t$.
\end{itemize}

\subsection{Motivating applications}


Our main motivation is the control of computer and software systems.
These systems are composed of multiple interconnected \emph{containers}, where a container can be a cluster of servers or a modular software system, and admission control mechanisms are commonly employed to optimize performance.
In the literature, containers are usually modeled via product-form queueing networks (for tractability) or layered queueing networks \cite{MOL,Ju09}, which justifies our modeling approach.
In serverless computing, for instance, users of the serverless platform can control the overall number of simultaneous requests that can be processed in a cluster of servers (each with its own queue) at any given time.
In Knative, a Kubernetes-based platform to deploy and manage modern serverless workloads that is used among others by Google Cloud Run, admission thresholds are set via the \texttt{container-concurrency-target-default} global key \cite{KnativeAdmission} and the upper limit on the number of jobs that can be active running at the same time, i.e., $S$, can be controlled via the \texttt{max-scale-limit} global key.
In Kubernetes, an open-source system for the management of containerized applications,
admission controllers are configured via the \texttt{--enable-admission-plugins} and \texttt{--admission-control-config-file} flags and can be leveraged in case the pod (or application) is requesting too many resources.
%
%

Because of the complex relationships among containers,
which can also be nested in multiple layers,
i) a detailed knowledge of the current \emph{state} is expensive to obtain at any point in time
and ii) the internal container structure is also subject to estimation errors and may vary over time~\cite{Ju22}.
This leads us to our learning model, which is meant to capture both of these aspects: we do not know the network topology, routing probabilities and service rates as well as the current ``state''.

  \section{MDP Model}

  This section is dedicated to the construction of an  MDP model of the system as well as an artificial aggregated MDP  that is equivalent to the original  MDP under its stationary regime.
  Note however that the learning algorithm constructed in the following only interacts with the original system, non-aggregated. The aggregated system is only used for the performance analysis of the algorithm.

  \subsection{Original MDP}

Let us model the problem as an MDP,
  $M^o=(\mathcal{X}^o,\mathcal{A}^o,P^o,r^o)$, where the super-index $o$ stands for ``original'' throughout the paper. We first use uniformization to see  the process in discrete time.
The uniformization constant $U$ is lower bounded by the sum of the rates: $U \geq \lambda+\sum_{i=1}^N \mu_i^o$.
Thus, the time steps, which will be indexed by $t$, follow a Poisson process with rate $U$, and events (arrivals, services, routings and control actions) can only occur at these times. In the following $1/U$ will be seen as one time unit. 

\begin{itemize}
\item The state space $\mathcal{X}^o$ is the set of all tuples $(x_1,\ldots,x_N)$ given by the number of jobs $x_i$ in each queue $i$.

\item The action space is $\mathcal{A}^o:=\{\rm 0,1\}$ where $0$ stands for rejection and $1$ for admission. 

\item The transition matrix $P^{o}$ is simply constructed by using the routing matrix $L$, the arrival rate $\lambda$ and the service rates $(\mu_i^o)_i$.


\item The mean rewards $r^o$ are constructed  from the cost function.
  The immediate cost for each  state-action pair $(\bx,a)$, is  Bernoulli distributed.
  It is decomposed into: \\
  - a deterministic part, $\frac{1}{U}(\hcost \sum_{i=1}^N x_i)$ (each present job incurs a cost $\hcost$ per time unit),\\
  - and a stochastic part, $\rcost(1-a) {\bf 1}_{\mbox{\tiny job-arrival}}$ (if a job arrives and the action is reject).

  To be consistent with the learning literature, where  rewards are used instead of costs,  we first define $r_{\max} := \frac{\lambda \rcost  + \hcost S}{U}$ and 
  for each  state-action pair $(\bx,a)$, the reward are Bernoulli distributed with expected value 
 \begin{equation}
   \label{eq:o-reward}
   r^o(\bx,a):=r_{\max}-  \frac{\lambda \rcost (1-a) + \hcost s}{U} =  \frac{\lambda \rcost a + \hcost (S-s) }{U},
 \end{equation}
\end{itemize}
where $s := \sum_{i=1}^N x_i$.

Let $\Pi^o:=\{\pi:\mathcal{X}^o\to\mathcal{A}^o\}$ denote the set of {stationary and deterministic policies}.
A stationary {\it policy} $\pi$ is  a deterministic function  from $\mathcal{X}^o$ to $\mathcal{A}^o$.  

Then, the MDP evolves under $\pi$ in the standard Markovian way.
At each time-step $t$, the system is in state~$\mathbf{x}_t$, the controller  chooses the action $a_t = \pi(\mathbf{x}_t)$ and receives a random reward whose expected value is  $r^o(\mathbf{x}_t,a_t)$,
and the system moves to state $\mathbf{x}'$ at time $t+1$ with probability $P^o(\mathbf{x}' \mid \mathbf{x}_t,a_t)$.
The objective function is to minimize the long run average cost.

The average reward induced by policy $\pi$ is:
\begin{equation}
\label{rho_M_pi'}
\gain^o(M^o,\pi) :=  \lim_{T\to\infty} \frac{1}{T}\sum_{t=1}^T \mathbb{E}[r^o(\mathbf{x}_t,\pi(\bx_t))].
\end{equation}

        An optimal policy $\pi^*$ for the original MDP achieves the best average reward
        $\gain^o(M^o,\pi^*) = \sup_{\pi  \in \Pi^o}\gain^o(M^o,\pi).$

        \subsection{Aggregated model}

        Let us define an aggregated  MDP $M=(\mathcal{S},\mathcal{A},P,r)$ where  the network is replaced by a  single queue.

\subsubsection{Norton equivalent queue}
In this subsection, let us consider the system  defined in Section \ref{sec:network} without control (all jobs are admitted).

The stationary measure of the  network can be connected to the stationary measure of a birth-and-death process via Norton's theorem of queueing networks \cite{Norton}, also known in the literature as Flow Equivalent Server (FES) method~\cite{Bolch}. Towards this purpose,
let $v_i$ denote the average number of visits at queue $i$ relative to queue $0$.
Set $v_0=1$ and let  $(v_1,\ldots,v_N)$ be the unique solution of $v_i=\sum_{j=0}^N L_{j,i} v_j$, for all $i=1,\ldots,N$.
Then,
when containing $S$ jobs, the vector of the number of jobs in each queue  forms a continuous-time Markov chain with stationary measure \cite{Bolch}
\begin{align}
\label{eq:productform}
\sm^o(\mathbf{x}) = \frac{1}{G(S)} \prod_{i=0}^N \left(\frac{v_i}{\mu_i}\right)^{x_i},
\end{align}
for all $\mathbf{x}\in\{\mathbf{x}\in\mathbb{N}^N:|\mathbf{x}|\le S\}$,
where $G(S)$ is a normalization constant and $|\cdot|$ denotes the $L_1$ norm.

The construction of our equivalent queue works as follows:
\begin{enumerate}
 \item Given the closed Jackson network above, consider the (closed) network where queue~$0$ is short circuited (this means set $\mu^o_0 = \infty$) and let $\mu(s)$ denote the \emph{throughput}\footnote{The throughput of a closed Jackson queueing network with $s$ jobs
is the rate at which jobs flow at a reference queue (queue 0 in our case) and is defined by $\mu(s):=\frac{G(s-1)}{G(s)}$ where $G(s)$ is the normalizing constant appearing in the product-form expression~\eqref{eq:productform} of the stationary measure~$m^o$.} of the network with $s$ jobs in total (see Figure~\ref{fig:Norton} for an illustration).

\item Consider the original network where  all queues except $0$ are all replaced by a \emph{single} queue that operates with rate $\mu(s)$ if it contains $s$ jobs.
 \item Then,
 \begin{align}
\sum_{\mathbf{x}: |\mathbf{x}|=s}  \sm^o(\mathbf{x})
 = \sm(S-s,s), \quad \forall s=0,\ldots,S
 \end{align}
 where
 $\sm(S-s,s)$, for all $s=0,\ldots,S$, is the stationary measure of the reduced network with two queues.
\end{enumerate}

\vspace{-0.0cm}
\begin{figure}[hbtp]
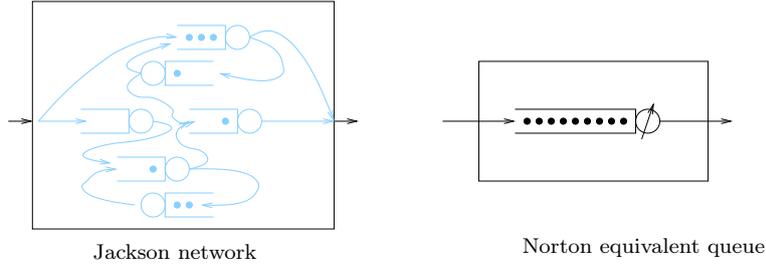

  \centering
      \include{nortoncode.tex}
      \vspace{-0.6cm}
  \caption{Illustration of Norton Equivalence theorem.}
  \label{fig:Norton}
\end{figure}

We remark that
$\sm(S-s,s)$ is indeed the stationary measure of a birth-and-death process with birth rate $\lambda \mathds{1}_{s< S}$ and death rate $\mu(s)$, a fact that will be key in the regret analysis of our learning algorithm.
In particular, we will use the following lemma, which provides some known properties about the throughput function $\mu(s)$~\cite{kameda}.

\begin{Lem}
The throughput function $s\mapsto \mu(s)$ is increasing, concave and bounded by
$\mu_{\max}:=\sum_{i \leq N} \mu_i^o$.
\label{lem:EqRates}
\end{Lem}
%

The throughput bound $\mu_{\max}$ can be significantly improved \cite{Bolch} but this will not change the structure of our results.

        \subsubsection{Aggregated MDP}

        Notice that, in the original MDP $M^o$, the rewards do not depend on the state but only on the number of jobs in the network, therefore, the Norton equivalent queue can also be used to construct an equivalent MDP.

        Define the simplified equivalent MDP $M=(\mathcal{S},\mathcal{A},r,P)$.
        \begin{itemize}
        \item 
The state space $\mathcal{S}= \{0,\ldots,S\} $ consists of all possible numbers of jobs in the queueing network. We denote by $S':=S+1$ the number of states of the aggregated MDP.
\item The actions are the same as for the original MDP: $\mathcal{A}=\mathcal{A}^o=\{\rm 0,1\}$ (reject or accept).
\item  The original reward in \eqref{eq:o-reward} does not depend on the precise position of the jobs in the network but only  on their number. Therefore for $s \in \mathcal{S}$ and $a \in \mathcal{A}$, we can define  the expected reward as
\begin{equation}
r(s,a)=\frac{\lambda \rcost a +  \hcost (S-s)}{U}
 \label{eq:reward}
\end{equation}

  \item
The transition probabilities  $(P^\pi)$ are defined as follows.

    Let $\pi$ be a policy (a function from $\mathcal{S} \to \mathcal{A}$) on $M$. By convention, $\pi$ will  also be seen  as a policy in the original MDP $M^o$ using the natural extension, i.e., if $\mathbf{x} \in \mathcal{S}^o$, then
    $\pi(\mathbf{x}) := \pi( |\mathbf{x}|)$. We can now define the  transition matrix for policy $\pi$ as the transition matrix in the aggregated MDP $M$ under the stationary measure $\sm^{o,\pi}$:
 \begin{equation}
  P(s' \mid s,\pi(s))=\sum_{\mathbf{x}, |\mathbf{x}|= s}\; \sum_{\mathbf{y},|\mathbf{y}| =s'} \frac{\sm^{o,\pi}(\mathbf{x})}{\sm^\pi(s)}  P^o(\mathbf{y} \mid \mathbf{x},\pi(\mathbf{x})),
  \label{eq:AggTrans}
 \end{equation}
 where $\sm^\pi(s)=\sum_{\mathbf{x}, |\mathbf{x}|= s} \sm^{o,\pi}(\mathbf{x})$ is the equivalent stationary measure.
Under this construction, these probabilities are those for the Norton equivalent queue. Also, notice that the equivalent stationary measure $\sm^\pi(s)$ is also  the stationary measure of the Norton equivalent queue  with transition matrix $P^\pi$ under policy $\pi$.
\end{itemize}

Let $\Pi:=\{\pi:\mathcal{S}\to\mathcal{A}\}$ denote the set of {stationary and deterministic policies}.

   \begin{Def}
     The average gain induced by policy $\pi$ is:
     \begin{equation}
\label{rho_M_pi}
\gain(M,\pi) :=  \lim_{T\to\infty} \frac{1}{T}\sum_{t=1}^T \mathbb{E}[r({\bs}_t,\pi(\bs_t))].
\end{equation}
    \end{Def}

    The   optimal policy $\pi^*$ achieves
    \begin{equation}\label{eq:optimal gain}
      \gain(M,\pi^*) := \gain^*(M) := \sup_{\pi  \in \Pi}\gain(M,\pi).
    \end{equation}

    \subsection{Comparison between both MDPs}

    It should be clear that the original MDP $M^o$ has a greater set of policies than the aggregated MDP $M$ because it has more states.
    Therefore,  $\gain^o(M^o,\pi^*) \geq \gain^*(M)$. However, if we only consider the set of policies in the original MDP $M^o$ that take the same action (reject or accept) in all the states with the same total number of jobs, then optimal gains coincide.
    More precisely, let $\Pi^o_{sum}$ be the subset of policies  in $M^o$
    such that for all $\pi \in \Pi^o_{sum}, \pi(\mathbf{x}) = \pi(\mathbf{y}) $ if $|\mathbf{x}|= |\mathbf{y}|$.
    Then, the stationary measure on $M^o$ under any policy $\pi$ in  $\Pi^o_{sum}$, and the stationary measure under $\pi$ on $M$
    satisfy
    $\sum_{\mathbf{x}, |\mathbf{x}|= s} \sm^{o,\pi}(\mathbf{x}) = \sm^\pi(s)$.
    Therefore, we get for all  $\pi$ in  $\Pi^o_{sum}$,
    \begin{eqnarray*}
      \gain^o(M^o,\pi) & = & \sum_{\mathbf{x}} \sm^{o,\pi}(\mathbf{x}) r^o(\mathbf{x},\pi(\mathbf{x}))\\
                       & = & \sum_{s} \sum_{\mathbf{x}, |\mathbf{x}|= s} \sm^{o,\pi}(\mathbf{x}) r(s,\pi(s))\\
      & = &  \sum_{s} \sm^\pi(s)  r(s,\pi(s)) =  \gain(M,\pi).
    \end{eqnarray*}
    Now taking the maximum over all policies in  $\Pi^o_{sum}$ yields
    \[ \max_{\pi  \in \Pi^o_{sum}}\gain^o(M^o,\pi) = \gain^*(M).\]

    As for learning when the full state is not observable, the best one can hope for is to learn  $\max_{\pi  \in \Pi^o_{sum}}\gain^o(M^o,\pi)$, so we will consider the regret with respect to the  optimal gain $\gain^*(M)$, in the following.

\subsection{Bias and diameter}

In the following, we will heavily use  the bias \cite{puterman-1994} of a policy on $M$ (although it has  no intuitive meaning relative to the original MDP)  as well  the \emph{diameter} of the aggregated MDP.

\begin{Def}[Diameter of an MDP]
Let $\pi: \mathcal{S} \to \mathcal{A}$ be a stationary policy of any MDP $M$ with initial state $\bs$.
Let $T\left(\bs'|M,\pi,\bs\right):=\min\{t\ge 0: \bs_t=s'|\bs_0=\bs\}$ be the random variable for the first time step in which $\bs'$ is reached from $\bs$ under $\pi$.
Then, we say that the \emph{diameter} of~$M$ is
$$
D(M):=\max_{\bs\neq \bs'} \min_{\pi:\mathcal{S}\to \mathcal{A}} \mathbb{E}\left[ T\left(\bs'|M,\pi,\bs\right)\right].
$$
\end{Def}

Again, let us point out  that we will only consider the diameter on the aggregated  MDP for the computations, as it is needed to control the bias terms of the aggregated  MDP (see Appendix \ref{sec:aggregated}). We will  never need to consider the bias or the diameter of the original MDP.

\subsection{Reinforcement learning}

Here, we consider a learner that can observe the  arrivals and departures of jobs in the original MDP
and makes admission decisions for each arriving job.

\subsubsection{What does the learner know?}

\begin{itemize}
\item As mentioned earlier, the learner can observe the external events: arrivals and departures of jobs. This implies that at any discrete time-step  $t$, the total number of jobs in the system, $s_t$ is known to the learner and will be seen as the partially observed {\it state}.
\item The expected cost in state-action pair $(s,a)$  is unknown as it depends on the unknown parameter $\lambda$ (see \eqref{eq:reward}). However, the parameters  $\rcost$, $\hcost$ and the uniformization constant $U$ are known, and the learner knows how the cost depends on $\lambda$, which will be important for the definition of the confidence regions.  We will often use an upper bound on the difference of rewards between two neighboring  states $\delta_{\max}:=\rcost +\frac{\hcost}{U}$. We will use $\delta_{\max}$ instead of  $r_{\max}$ in the following derivation of the regret, 
  as $\delta_{\max}$ does not depend on $S$ and this will help us gain a factor $S$ in the regret bound.
  We will also make the assumption that the learner knows the reward function up to the  actual value of the arrival rate $\lambda$, that must be learned.

\item The learning algorithm knows $T$, the number of time steps where it can take observations and actions. This is not a strong requirement as one can  make the algorithm oblivious to $T$ by using a classical doubling trick on $T$.
\end{itemize}

\subsubsection{Regret}

\begin{Def}[Regret]
The \emph{regret} at time~$T$ of the learning algorithm $\mathcal{L}$ is
\begin{align}
\label{def:regret}
{\rm Reg}(M,\mathcal{L}, T):= T g^*(M)-\sum_{t=1}^T r_t.
\end{align}
\end{Def}
Here, $g^*(M)$ is the optimal gain defined in~ \eqref{eq:optimal gain}.
The reward $r_t$ is the reward of the state visited at time $t$ under the current policy used by the learning algorithm.


%

\section{Learning algorithm}
\label{sec:algo}

\subsection{High-level description of the proposed algorithm}
\label{ssec:description}
Our algorithm is  {\it episodic},  {\it model-based} and {\it  optimistic}.
More precisely, the interactions of the learner with the MDP $M^o$  are decomposed into {\it episodes}. In each episode $k$, of duration $[t_k, t_{k+1}-1]$, one admission policy $\pi_k$ is used to control the network and the learner observes the system (arrivals and departures) while collecting rewards under $\pi_k$.
At the end of the episode, the estimation of the true transition probabilities and rewards (the {\it model}),  $\hat{p}_k$ and $\hat{r}_k$ respectively, as well as the {\it confidence region} $\mathcal{M}_k$  are updated using the samples collected during the episode. This gives $\hat{p}_{k+1}$, $\hat{r}_{k+1}$  and $\mathcal{M}_{k+1}$.
The next policy $\pi_{k+1}$ is the best policy for the best MDP inside the confidence region $\mathcal{M}_{k+1}$ ({\it optimism}).

In our case with partial observations,
the number of jobs  at  time $t$, $(s_t)_{t \leq T}$ is not Markovian, therefore it does not provide enough information to make good estimates on the underlying MDP. Instead, we collect a set $\{s_{1}, \ldots, s_{\tm }\}$ of observations and try to learn using  this extended information.
If $\tm $ is well chosen, i.e., larger than the mixing time of the MDP, then each  subsequence  $s_i, s_{i+\tm },  s_{i+2\tm }, \ldots$ forms an ``almost'' independent sequence and therefore can be used for statistical estimations.

Our learning algorithm is based on the following idea. It can be seen as a collection of $\tm $ learning algorithms $\mathcal{L}_1,\ldots, \mathcal{L}_{\tm }$, using respectively the subsequence $(s_{i+k\tm })_{k\in \N}$ of observations, which are called  {\it modules} in the following.
Each learning module  $\mathcal{L}_i$ behaves similarly as the classical optimistic algorithm described above. There are no  interactions between modules
except for the number of visits that contributes to  the construction of the global confidence region, as detailed  in Section \ref{ssec:confidence}.
The main technical difficulties  in the control of the behavior of the algorithm are:
\begin{enumerate}
\item The observations used by  the learning modules $\mathcal{L}_1,\ldots, \mathcal{L}_{\tm }$ are not independent of each other, so one must be careful in assessing the interplay between the modules.
\item For each learning module $\mathcal{L}_i$, its  sequence of observations  $s_{i+\tm }, s_{i+2\tm },  s_{i+3\tm }, \ldots$ is not really stationary and independent, but only weakly correlated.
\end{enumerate}

\subsection{Number of modules: \texorpdfstring{$\tm$}{taumix}}
\label{ssec:t*}

Let us first give a more precise definition of the $\tm$ modules, where the number $\tm$ is yet to be chosen carefully. At the beginning of the algorithm, each time-step $t$ is attributed a module $\module_t$, so that these modules form a partition of the time-steps. For $0 \leq t\leq \tm-1$, the module $\module_t$ is defined in the following way: first $t \in \module_t$, then we wait $\tm$ steps to add the next time-step to that module, so that $t, t+\tm, t+2\tm, \ldots \in \module_t$, until time-step $T$ is reached. More formally one can identify, $\module_t = t \mod \tm$.

The number of modules $\tm$ is chosen using the following construction.
Let us consider the  original MDP  under any policy $\pi$,   with stationary measure $\sm^{o,\pi}$.
There exists $C>0$, $\mrate\in(0,1)$ such that:
\begin{equation}
\max_{\pi \in \Pi}\sup_{x_0\in\mathcal{S}} \left\|\mathbb{P}^{o,\pi}_{x_0}(x_t=\cdot)-\sm^{o,\pi} \right\|_{TV}\leq C\mrate^t \quad \forall t>0,
 \label{eq:mixing}
\end{equation}
where $\mathbb{P}^{o,\pi}_{x_0}(x_t=\cdot)$ is the distribution of the state at time $t$ under policy $\pi$ in the original MDP, with starting state $x_0$.
Let us then define
\begin{equation} \tm:= \lceil 5 \log T /\log \mrate^{-1} \rceil.
\end{equation}
The reason for this precise choice will appear in the analysis of the regret (see Section \ref{sec:regret}) but
the general idea behind this choice comes from  Lemma \ref{lem:Russo} given in appendix,  that basically says that after $\tm$ steps, the correlation between the state at time $t$ and the state at time $t + \tm$, under any policy,  is smaller than $C' \mrate^{\tm}$, where $C'$ is a constant.

The  fact that the number of modules used by the algorithm depends on $\mrate$  can be seen as a weakness of our approach because it means that the learner needs to know {\it a priori} a bound on the mixing time of the unknown MDP. This point will be addressed in Section~\ref{sec:mixing}.



\subsection{UCRL-M: learning with \texorpdfstring{$\tm $}{taumix} modules}

Algorithm UCRL-M (Upper Confidence Reinforcement Learning with several Modules) is given in Algorithm~\ref{algo:UCRL-M}.
First, the algorithm initializes the different  modules. Here, for each episode $k$ and  module $\module$, it computes the empirical estimates
of the reward and probability transition as in \eqref{eq:estimates_r_p} and \eqref{eq:estimates_r_p2}.
Then, it applies Extended Value Iteration (EVI) (Appendix~\ref{sec:EVI})
to find a policy $\tilde{\pi}_k$ and an optimistic MDP $\tilde{M}_k \in \mathcal{M}_k$ according to \eqref{ucrl2_ineq}.
Finally, to explore the MDP at episode $k$, it first iterates on the MDP over $\tm $ time-steps and discards these samples (ramping phase) to start the observations from the stationary distribution of the current policy. This phase is necessary to guarantee that observations within a module are nearly independent.
Afterward, UCRL-M  explores the true MDP with the optimistic policy $\tilde \pi_k$ and updates the empirical estimates with its observations.

The episode ends when the stopping criterion \eqref{eq:stop} is met. The next optimistic policy for the episode $k+1$ is found with respect to the observations inducing the confidence region $\mathcal{M}_k$ that is built using all modules (see \eqref{confP}).

 \begin{algorithm}[H]
    \label{algo:UCRL-M}
  \caption{The {\sc UCRL-M} algorithm.}
  \KwIn{ $\mathcal{S}$ and $\mathcal{A}$.}
  Set $t=0$, $k=0$;\\
  \While{$t\leq T$}
  {
   {\bf Initialize} episode $k$ with $t_k:=t$\\
    {\bf Compute} for all $(s,a)$ the modules $\module_k(s,a)$ according to \eqref{eq:module};\\
    {\bf Compute} the confidence region $\mathcal M_k$ as in \eqref{confP};\\
    {\bf Find} a policy $\tilde{\pi}_k$ and an optimistic MDP $\tilde{M}_k \in \mathcal{M}_k$ with ``Extended Value Iteration'' such that
\begin{equation}
g (\tilde{M}_k, \tilde{\pi}_k)\ge \max_{M_k\in {\mathcal{M}}_k}\max_\pi g (M_k, {\pi}) - \frac{\delta_{\max}}{\sqrt{t_k}}.
\label{ucrl2_ineq}
\end{equation}\\
     {\bf Ramping phase ($\Phi$):} Iterate the MDP with policy $\tilde \pi_k$ for $\tm $ time-steps, discard the observations and set $t:=t+\tm $.\\
    {\bf Exploration:} \While{ $
  \quad V_k^{(\module_t)}(s_t,\tilde \pi_k(s_t)) < \max \{1, N_{t_k}^{(\module_t)}(s_t,\tilde \pi_k(s_t))\},
$ }{
     \begin{itemize}
      \item[1.] Choose action $a_t = \tilde{\pi}_k(s_t)$;
      \item[2.] Observe $s_{t+1}$;
      \item[3.] Update $V_k^{(\module_t)}(s_t,a_t):=V_k^{(\module_t)}(s_t,a_t)+1$;
      \item[4.] Set $t:=t+1$.
     \end{itemize}
}

}
\end{algorithm}

\subsection{Confidence region}
\label{ssec:confidence}

As mentioned earlier, the learning algorithm relies on the ``Optimism in face of uncertainty'' principle. Here,  we provide the explicit construction of  a confidence region  $\mathcal M_k$ based on the observations, which depends on the visit counts.
For each state-action pair $(s,a)$ and each module $\module$, let $N_{t_k}^{(\module)}(s,a)$  be the cumulative number of visits to $(s,a)$  at all times $t = \module \mod \tm$ smaller than $t_k$, and excluding the visits during the ramping phases $\Phi$ (see the UCRL-M algorithm).

We also define the {\it most frequent module} for each state-action pair $(s,a)$: Let  $\module_k(s,a)$ be a module with the highest visit count until episode $k$,
\begin{equation}\label{eq:module}
 \module_k(s,a) \in \arg\max_{\module}{ N_{t_k}^{(\module)}(s,a)},
\end{equation}
so that for this module, the empirical observations are the most accurate, and we can relate the number of observations for this module to the total number of visits $N_{t_k}(s,a)$ of the pair $(s,a)$ with the inequality: $ N_{t_k}^{(\module_k(s,a))}(s,a) \geq \frac{1}{\tm}N_{t_k}(s,a)$.

To define the confidence region $\mathcal{M}_k$, first define $\hat{r}_k^{(\module)}$ and $\hat{p}_k^{(\module)}$ the empirical reward and transition estimates in module $\module$:
\begin{align}
\label{eq:estimates_r_p2}
\hat{p}_k^{(\module)}(s'|s,a)&:=
\frac
{\sum_{t=1}^{t_k-1} \mathds{1}_{\{s_t=s,a_t=a,s_{t+1}=s',\module_t=\module\}}\mathds{1}_{\croc{t\notin \Phi}}}
{\max\left\{1,N_{t_k}^{(\module)}(s,a)\right\}}
\\
\label{eq:estimates_r_p}
\hat{r}_k^{(\module)}(s,a)&:=
\hat{p}_k^{(\module)}(s+1|s,a)
\rcost + \hcost\frac{S-s}{U},
\end{align}
where $\Phi$ is the set of the time steps in the ramping phases defined in the algorithm. $\mathcal{M}_k $ is the confidence set of MDPs whose rewards $\tilde{r}$ and transitions $\tilde{p}$ satisfy:
\begin{equation}
 \label{confR}
 \forall (s,a), \quad \left|\tilde{r}(s,a)-\hat{r}_k^{(\module_k(s,a))}(s,a) \right|\leq \delta_{\max}\sqrt{\frac{2\log\left(2At_k \right)}{\max\croc{1,N_{t_k}^{(\module_k(s,a))}(s,a)}}};
\end{equation}
\begin{equation}
\label{confP}
\forall (s,a), \quad \|\tilde p(\cdot\mid s,a)-\hat{p}_k^{(\module_k(s,a))}\prt{\,\cdot\mid s,a}\|_1 \leq \sqrt{\frac{8\log\prt{2 At_k}}{\max\{1,N_{t_k}^{(\module_k(s,a))}(s,a)\}}}.
\end{equation}
Notice that for each state-action pair $(s,a)$, we only need the empirical reward and transition estimates for the module $\module_k(s,a)$: this means that the confidence region $\mathcal M_k$ is built from the comparison between modules from \eqref{eq:module}, and we do not build a specific confidence region for each module.

The algorithm finds the best optimistic MDP and policy within this confidence set, and executes the policy on the true MDP until the stopping criterion is met, that is when for any module $\module$ the number of visits $V_k^{(\module)}(s,a)$ in the current episode of a state-action pair $(s,a)$ reaches the number of visits of this pair and module until time $t_k$.  More formally, if at episode $k$ we choose the policy $\tilde \pi_k$, then the stopping criterion gives the following guarantee:
\begin{equation}
 \forall (s,\module) \quad V_k^{(\module)}(s,\tilde \pi_k(s)) \leq \max \{1, N_{t_k}^{(\module)}(s,\tilde \pi_k(s))\}.
 \label{eq:stop}
\end{equation}

\subsection{Time complexity of UCRL-M}

\begin{proposition}
  The time complexity of UCRL-M is $O(K S \tm  + Kt_{\rm evi} + T)$, where $K$ is the number of episodes and $t_{\rm evi}$ the time complexity of extended value iteration.
  Furthermore, $\E (K) = O(\log T)$.
\end{proposition}

\begin{proof}
  The time complexity of lines 5 and 6 is $O(KS\tm )$. The complexity of line 7 is $O(Kt_{\rm evi})$.
  The complexity of line 8 is $O(K\tm )$. The complexity of line 9 is $O(T - K\tm )$, the number of useful observations.
  As for the expected number of episodes,  $\E K = O(\log T)$ because of the doubling trick used to end the episodes (see \cite{jaksch-2010} for example).
\end{proof}

Note that the total number of useful samples (excluding the steps made during the ramping phases ) is $T- K\tm $, and  each module uses
$\frac{T - K\tm }{\tm }$ samples.
As for the time complexity of EVI, each iteration of EVI is $O(S^3)$ and the number of iterations depends on the starting point and is more difficult to estimate.
In total, the time complexity does not really depend on $\tm$ or $t_{\rm evi}$ that only appears at the beginning of each episode, and the  number of episodes is small w.r.t. $T$.

\section{Regret of {\sc UCRL-M} }
\label{sec:regret}
\subsection{Main result}

Let us recall that $S$ is the global bound on the number of jobs, $S'=S+1$ is the number of states, $\rcost$ is the rejection cost, $\hcost$ is the unit-time holding cost and
$D$ is the diameter of the aggregated MDP.
Also, $\sm^{\pi^{\max}}(s)$ is the stationary measure in the aggregated MDP under the policy that accepts all jobs,
$\mrate$ is defined in Section \ref{sec:algo}  and
$\mu(i)$ is the service rate in the aggregated MDP when $i$ jobs are in the system.

Define the constant $C_1:=\prod_{i=1}^{i_0-1}\frac{\mu(i_0)}{\mu(i)}\geq1$, where $i_0$ is chosen such that $\mu(i_0)>\lambda$. Such a $i_0$ exists because the unconstrained network is assumed to be stable (see Section~\ref{sec:network}) regardless of $S$. Hence, the flow equivalent queue is also stable regardless of~$S$. Define also $C_2:=\frac{(\lambda \rcost + \hcost)C_1}{\mu(1)\prt{1-\lambda/\mu(i_0)}}$.

\begin{The}
 Let $M \in \mathcal M$. Define $Q_{\max}:=\left(\frac{10C_2 (S')^2}{\sm^{\pi^{\max}}(S)}\right)^2 \log \left(\left(\frac{10C_2 (S')^2}{\sm^{\pi^{\max}}(S)}\right)^4\right)$. Define also the constant $ \kappa=228( \rcost + \frac{\hcost}{U}) \frac{U}{\mu(1)} C_1 {\prt{1-\sqrt{\frac{\lambda}{\mu(i_0)}}}^{-3}}$. For the choice $\tm =5\frac{\log T}{\log 1/\mrate}$, and $A=2$, we have:
 \begin{equation}
  \E\brac{\Reg(M, \mbox{UCRL-M},T)}\leq  \kappa\log\left(2T\right) \sqrt{T\log^{-1}(1/\mrate)} + R_{\rm LO},
 \end{equation}
where $R_{\rm LO}:= 138r_{\max}D^2 \max\croc{Q_{\max},T^{1/4}}\frac{\log^4\prt{4T}}{\log^2 1/\mrate}$ is a lower order term of the regret.
\label{th:main}
\end{The}

Before diving into the proof, which involves many technical points, let us comment on our result.
In contrast with most bounds from the literature,
the most remarkable point is that both the diameter and the size of the state space do not appear in our bound.
These are both replaced by $\log^{-1/2} (1/\mrate)$.

Although we do not know any explicit bounds on $\mrate$ for all possible networks, it is quite reasonable to predict that $\log^{-1/2} (1/\mrate)$ can be  of order $\sqrt{S}$. In fact, this can be shown for acyclic networks as well as for hyper-stable networks as it will be shown in Section \ref{sec:mixing}.

This implies that the regret of UCRL-M  is $\tilde O(\sqrt{ST})$, which is a major improvement over the best bound for general MDPs, namely  $\tilde O(\sqrt{DSAT})$.
This  further confirms the fact that exploiting the structure of the learned system actually leads to more efficient algorithms as well as tighter analysis of their performance.

\subsection{Outline of the proof}\label{ssec:outline}

To compute the expected regret $\mathbb{E}[{\rm Reg}]$, we will mainly follow the strategy from \cite[Section~4]{jaksch-2010}. First, we deal with the regret term corresponding to the initialization phase of each episode, which depends in the number of episodes. Then, for each episode $k$, we consider the case where the true MDP $M$ does not belong to the confidence region $\mathcal{M}_k$, and use concentration inequalities along with the independence Lemma~\ref{lem:Russo} to show that this regret term will remain low. Then, we consider the case where the true MDP belongs to the confidence region, and  for each episode, we split the regret into relevant comparisons. Here, we expose terms depending on the difference of rewards and transitions between the true and optimistic MDPs, terms depending on the difference of biases, a term depending on the number of episodes and a term coming from the the computation of the optimistic policy and MDP with EVI.

To achieve the first split, we need to define: $R_k^{(\module)}(s):= \sum_{a}V_k^{(\module)}(s,a)(\mrate^*-r(s,a))$ the regret at episode $k$ induced by state $s$ in module $\module$, with $V_k^{(\module_t)}(s,a)$ the number of visit of $(s,a)$ during episode $k$ in module $\module$.
We split the regret into terms where the true MDP belongs to the confidence region, terms where it does not, and the terms from initializing the episodes:
\begin{equation}
 \mathbb E \brac{\rm Reg} \leq \mathbb E \brac{R_{\rm in}} +\mathbb E \brac{R_{\rm out}}
  +\E\brac{R_{\rm ramp}}
\label{eq:in-out}
\end{equation}
with $K$ the number of episodes and the regret where the MDP is in the confidence region being $R_{\rm in}:= \sum_\module\sum_{s}\sum_{k=1}^{K} R_k^{(\module)}(s) \mathds{1}_{M \in \mathcal{M}_k}$, and when it is outside $R_{\rm out}:= \sum_\module\sum_{s}\sum_{k=1}^K R_k^{(\module)}(s) \mathds{1}_{M \notin \mathcal{M}_k}$ and the regret of the ramping phases $R_{\rm ramp}=\sum_k \sum_{t=t_k}^{t_k+\tm -1}r(s_t,\tilde \pi_k(s,t))$.
Each term is then bounded as explained in Appendix \ref{app:proof}.

 \section{Controlling the regret bound parameter \texorpdfstring{$\mrate$}{rho}}
 \label{sec:mixing}

 The efficiency of UCRL-M is critically based on controlling $\tm $ and $\mrate$.
 In particular, Theorem \ref{th:main} says that the regret of UCRL-M depends on $W:=\log^{-1/2} (1/\mrate)$.

 \subsection{Bounds using mixing and coupling times}

 In Section \ref{sec:algo}, the number of modules  $\tm $ is defined as  $\tm :=5 \log T /\log \mrate^{-1}$.
 where $\mrate$ is such that
 \begin{equation}\label{eq:definition rho}
\max_\pi \sup_{x_0\in\mathcal{S}} \left\| \mathbb{P}^{o,\pi}_{x_0}(x_t=\cdot)-\sm^o(\pi) \right\|_{TV}\leq C\mrate^t \quad \forall t>0.
\end{equation}

Let us first recall classical results from Markov chain theory \cite{Peres-2008} relating $\mrate$ with the mixing and coupling time of a Markov chain.
Let us consider any Markov chain with transition matrix $P$ and stationary distribution $\sm$ (in our case, consider the Markov chain under the policy that attains the maximum in~\eqref{eq:definition rho}).
Let us define $d(t) := \sup_{x_0\in\mathcal{S}} \left\|\mathbb{P}_{x_0}(x_t=\cdot)-\sm \right\|_{TV}$.
Then, the {\it mixing time} of the chain is defined as
$t_{mix} := \min \{ t: d(t) \leq 1/4\}$.

A classical bound on $\mrate$ is then obtained by using the mixing time:

\begin{equation}
  \label{eq:mix}
  \mrate \leq \frac{1}{2^{t_{mix}^{-1}}}
\end{equation}

This implies that $W  \leq  \sqrt{t_{mix}  \log(2)} $.

Another bound on $\mrate$ can be obtained by using the coupling time.
The coupling time is $\tau_{x,y} := \min\{t: X_t = Y_t\}$. If $X_t$ and $Y_t$ are coupled and start at $X_0 = x$ and $Y_0 = y$ respectively.
Then,
$d(t) \leq \max_{x,y} \Pb(\tau_{x,y} > t)$.
By using Markov inequality, this implies that

\begin{equation}
  \label{eq:couple}
t_{mix} \leq 4  \max_{x,y}\E[\tau_{x,y}].
\end{equation}
Therefore, a  bound on the expected coupling time translates into a bound on $\mrate$.

\subsubsection{Acyclic networks}

In our model, if the queueing network is acyclic, then the coupling time is controllable because whenever a queue couples it stays coupled forever.

More precisely, since the total number of states in the network increases with the admission threshold, the  threshold policy under which the coupling time is the largest is when all jobs are admitted. Under this policy, by monotonicity, the coupling time is upper bounded by the coupling in an open network where all the $N$ queues have buffers bounded by $S$.
In this case, the coupling time has been studied in \cite[Theorem 5.3]{dopper-2006}, where the following result is proved in the stable case. Using our notation,
\begin{equation}
\max_{x,y}\E[\tau_{x,y}] \leq \sum_{i=1}^N \frac{U^2}{(\lambda_i^o + \mu_i^o)(\mu_i^o - \lambda_i^o)} S,
\end{equation}
where $U$ is the uniformization constant and $(\lambda_i)_{i\leq N}$ is the solution of the traffic equations.

According to Equation \eqref{eq:mix} and  \eqref{eq:couple}, this  induces the following bound  on the term $W$ in the regret:
\[ W \leq  \kappa_0 \sqrt{N S}, \]
where $\kappa_0$ is a constant: $\kappa_0= \max_i \sum_{i=1}^N \frac{U}{\lambda_i^o + \mu_i^o}$.

\subsubsection{Hyperstable networks}

This is another type of networks for which an explicit  bound on the coupling time exists.
A network is called {\it hyperstable} if for each queue $i$,
$\sum_j L_{ji} \mu_j^o + L_{0i}\lambda  < \mu_i^o$.

As in the acyclic case,  the threshold policy under which the coupling time is the largest is when all jobs are admitted. Under this policy, as for the acyclic case,  the coupling time is upper bounded by the coupling in an open network where all the $N$ queues have buffers bounded by $S$.

Coupling times of hyperstable networks with finite buffer queues have been studied in~\cite{anselmiGaujal-2014}, where the following bound is given (Theorem 2):
\begin{equation}
\max_{x,y}\E[\tau_{x,y}] \leq \kappa_2 N^2 S \sum_{i=1}^N \frac{\lambda_i^o}{\mu_i^o - \lambda_i^o},
\end{equation}
where $\kappa_2$ is a constant.
Using  Equation \eqref{eq:mix} and  \eqref{eq:couple}, this induces a similar bound on  the term $W$ in the regret:
\[ W \leq \kappa_3 N\sqrt{ S}, \]
where $\kappa_3$ is yet another constant.

\subsection{Making the algorithm oblivious to \texorpdfstring{$\mrate$}{rho} }
\label{ssec:oblivious}

By construction, the current version of UCRL-M uses  explicitly $\tm  = 5 \log T /\log \mrate^{-1}$ modules.
This can be a problem as it implies an {\it a priori} knowledge of $\mrate$, and of the mixing time (or at least an upper bound) of the network being learned.

These types of assumptions are sometimes made in the reinforcement learning literature. For example, the UCBVI algorithm \cite{UCBVI}  requires the knowledge of the diameter of the MDP being learned.

Here, we can patch  UCRL-M to make it oblivious to  $\mrate$ by making sure that $\tm  \geq  5 \log T /\log \mrate^{-1}$ for any large enough $T$. For example, one can chose $\tm  := \log^2(T)$, as it is asymptotically larger than the previous one. This patch adds a multiplicative $\log(T)$ term in the asymptotic bound of the regret given in Theorem~\ref{th:main}.

\section{Numerical experiments}
\label{sec:experiment}
\subsection{A multi-tier queueing network}

To assess the performance of UCRL-M, we rely on a standard multi-tier queueing network as displayed in Figure~\ref{fig:containers}.
The topology of this network is composed of three tiers. Namely, tiers 1, 2 and 3 represent the web, application and database stages of a typical web-application request.
Each tier is composed of multiple servers, each with its own queue.
After accessing the web tier, a request may either return back to the issuing user with probability $1-p$ or flow through the application and database tiers. This multi-tier structure is common in empirical studies of computer systems \cite{UR05} and is the default architecture of web applications deployed on Amazon Elastic Compute Cloud (EC2) \cite{3tier}.

\begin{figure}[h]
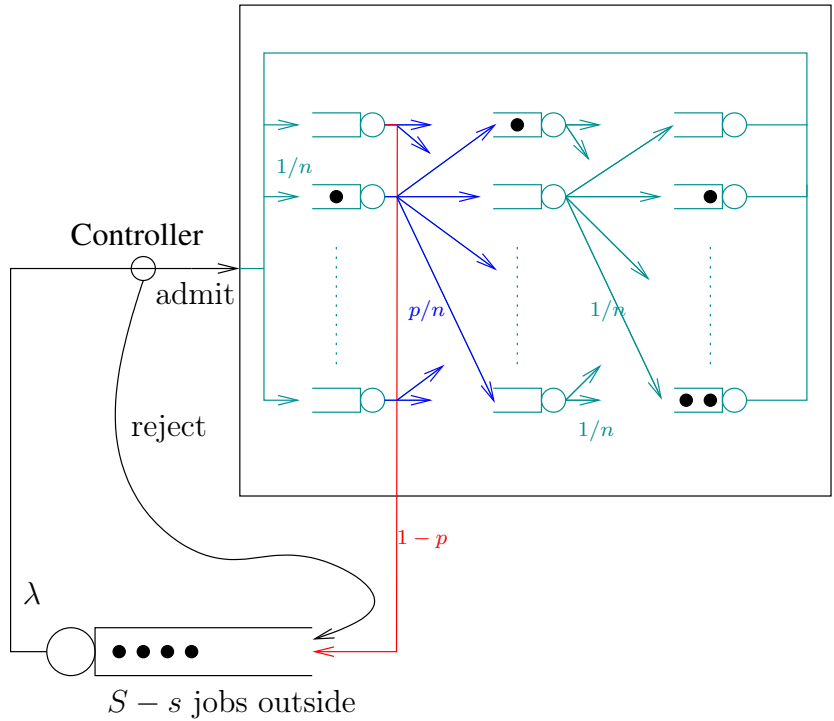

  \centering
  \include{containerscode.tex}
    \vspace{-0.65cm}
  \caption{A queueing network model with three interconnected tiers. Each tier contains $n$ queues and the total capacity is of $S$ jobs.}
  \label{fig:containers}
\end{figure}

This model may be studied as an example of the generic case described in Section~ \ref{sec:network}. Notice that given the routing from Figure~\ref{fig:containers}, the stability condition is met if $\frac{\lambda}{1-p}<\mu$, where $\mu=\mu_1^o=\ldots=\mu_{3n}^o$ is the service rate of the queues in the network.

\subsection{Regret of UCRL-M on the multi-tier queueing network}

We provide the performance of UCRL-M over the queueing network  described above when the number of queues per tier $n$ and the total number of jobs $S$ vary. In Figure~\ref{fig:regret}, we display the average regret over $66$ runs of the UCRL-M algorithm when $n$ varies, and with parameters scaling with $n$ to keep the systems proportionally comparable. More precisely, the  scaling in $S$ and $\mu$ is such that as the number of queues increases, the waiting time in each tier remains roughly identical for a job in each tier, and the scaling in the holding cost is also consistent with the increase of the number of jobs in the system. Notice that for our choice of parameters, the network is not stable, so that we use  the UCRL-M algorithm under  more general conditions than those assumed in Section~\ref{sec:mixing} and even in Section~\ref{sec:network}.


In Figure~\ref{fig:regret}, we remark that as we let the number of queues $n$ (and the number of jobs $S$) scales multiplicatively, the regret is increasing in $\log(S)$. Knowing that the dependency in $S$ of the regret bound from Theorem~\ref{th:main} mainly comes from $\mrate$, this is much slower than the square root bounds given  in Section~\ref{sec:mixing} (under strong assumptions). This can be interpreted as the bound of equation~\ref{eq:definition rho} being too large as it considers the mixing from the worst state, while in average it is more likely for the algorithm to mix from states that are visited the most, which are already close to stationary states.

\begin{figure}
\begin{minipage}{0.70\textwidth}
 \hspace{-0.08\textwidth}
  \includegraphics[width=1.0\textwidth]{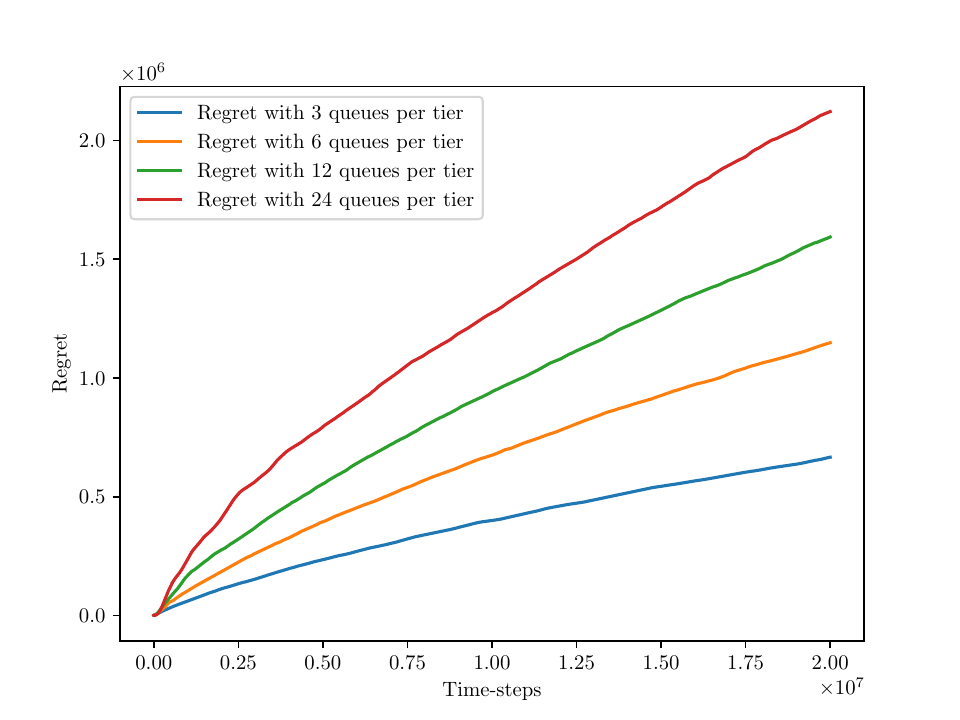}

\end{minipage}
 \hspace{-0.08\textwidth}
\begin{minipage}{0.1\textwidth}
\renewcommand{\arraystretch}{1.6}
\begin{tabular}{|@{\hspace{0.1cm}}c@{\hspace{0.1cm}}|c|}
    \hline
    Parameter & Value \\
    \hline \hline
    $S$ & $\frac{10}{3}n$  \\
    \hline
    $\lambda$ & $0.99$  \\
    \hline
    $\mu$ & $\frac{1}{n}$ \\
    \hline
    $\rcost$ & $10$ \\
    \hline
    $\hcost$ & $\frac{4}{3n}$ \\
    \hline
    $U$ & $\lambda+3n\mu$ \\
    \hline
    $p$ & $0.2$ \\
    \hline
    $\tm$ & $3$ \\
    \hline
\end{tabular}
 \end{minipage}
 \caption{Regret of the UCRL-M algorithm on the queueing network for different values of $n$ and scaling parameters.}
  \label{fig:regret}
  \vspace{-0.8cm}
 \end{figure}

 \begin{figure}[H]
 \centering
  \includegraphics[width=0.70\textwidth]{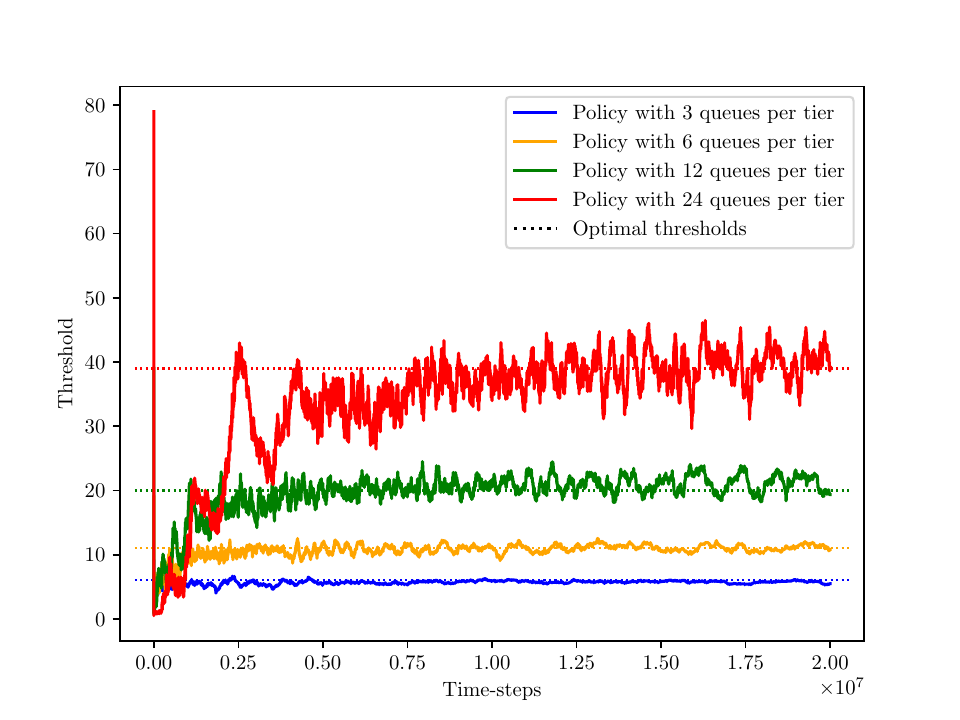}
  \caption{Average threshold chosen as a function of time.}
  \label{fig:policy}
 \end{figure}

In Figure~\ref{fig:policy}, we can see the convergence to the optimal threshold for each value of $n$. It suggests that the optimal threshold is scaling linearly with $n$, and that the convergence is slower as $S$ increases.

In the previous experiments, the number of modules is arbitrarily fixed to $\tm = 3$.
Now, we perform another experiment to observe the dependency of the regret in the choice of $\tm$ for this queueing system. The intuition is the following: as explained in the high-level description of the algorithm in Subsection~\ref{ssec:description}, UCRL-M could be compared to $\tm$ instances of UCRL2~\cite{jaksch-2010}, where all  modules but the best one is discarded at each episode. This best module runs on roughly $\frac{T}{\tm}$ time-steps, and its regret can be compared to $\frac{1}{\tm}$ times the expected regret of UCRL-M. With this intuition in mind,  we plot in Figure~\ref{fig:tmixReg} the regret of UCRL-M, where we rescaled both the regret and the time-steps by a factor $\frac{1}{\tm}$.

\begin{figure}[H]
\begin{minipage}{0.75\textwidth}
 \hspace{-0.08\textwidth}
  \includegraphics[width=1.0\textwidth]{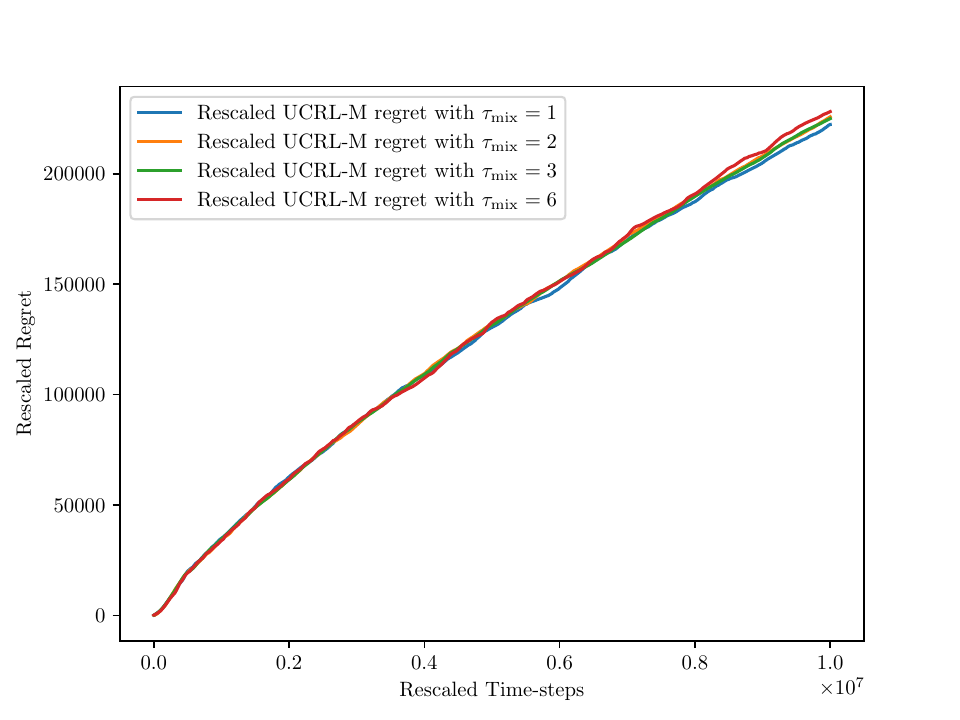}
\end{minipage}
 \hspace{-0.08\textwidth}
\begin{minipage}{0.1\textwidth}
\renewcommand{\arraystretch}{1.6}
\begin{tabular}{|@{\hspace{0.1cm}}c@{\hspace{0.1cm}}|c|}
    \hline
    Parameter & Value \\
    \hline \hline
    $S$ & $20$  \\
    \hline
    $n$ & $6$ \\
    \hline
    $\lambda$ & $0.99$  \\
    \hline
    $\mu$ & $\frac{1}{n}$ \\
    \hline
    $\rcost$ & $10$ \\
    \hline
    $\hcost$ & $\frac{4}{3n}$ \\
    \hline
    $U$ & $\lambda+3n\mu$ \\
    \hline
    $p$ & $0.2$ \\
    \hline
\end{tabular}
 \end{minipage}
 \caption{Rescaled regret of the UCRL-M algorithm on the queueing network for different values of $\tm$.}
  \label{fig:tmixReg}
 \end{figure}
Within the considered queueing model,  we notice that the modules do not seem to bring any practical upside because the regret is almost perfectly linear in the number of modules. However, they remain necessary to guarantee the correctness of the confidence sets and to get the theoretical bound on the regret given in Theorem~\ref{th:main}.

\section{Conclusion}\label{sec:conclusion}

In the context of queueing networks, we have shown that efficient learning in POMDPs is possible.
Provided that the learner's objective is to learn the optimal admission control policy, which is a problem appearing in a number of applications as discussed in Section~\ref{sec:network}, we have proposed UCRL-M, an optimistic  algorithm whose regret is independent of the diameter $D$, i.e., a quantity that appears in most of the existing regret analyses \cite{jaksch-2010} and that is exponential in the size of the space $S$ in most queueing systems.

While our result strongly relies on Norton's equivalent theorem, which only applies exactly to product-form queueing networks, our main perspective is that this type of  results under partial observations may be found in several other models from queueing theory. In fact, Norton's theorem has been
generalized to multiclass networks~\cite{norton_generalization} and
also used in the context of non-product-form queueing networks for approximate analysis~\cite{Bolch}.

\bibliographystyle{plain}
 \bibliography{biblio}


\begin{thebibliography}{36}
\ifx \bisbn   \undefined \def \bisbn  #1{ISBN #1}\fi
\ifx \binits  \undefined \def \binits#1{#1}\fi
\ifx \bauthor  \undefined \def \bauthor#1{#1}\fi
\ifx \batitle  \undefined \def \batitle#1{#1}\fi
\ifx \bjtitle  \undefined \def \bjtitle#1{#1}\fi
\ifx \bvolume  \undefined \def \bvolume#1{\textbf{#1}}\fi
\ifx \byear  \undefined \def \byear#1{#1}\fi
\ifx \bissue  \undefined \def \bissue#1{#1}\fi
\ifx \bfpage  \undefined \def \bfpage#1{#1}\fi
\ifx \blpage  \undefined \def \blpage #1{#1}\fi
\ifx \burl  \undefined \def \burl#1{\textsf{#1}}\fi
\ifx \doiurl  \undefined \def \doiurl#1{\url{https://doi.org/#1}}\fi
\ifx \betal  \undefined \def \betal{\textit{et al.}}\fi
\ifx \binstitute  \undefined \def \binstitute#1{#1}\fi
\ifx \binstitutionaled  \undefined \def \binstitutionaled#1{#1}\fi
\ifx \bctitle  \undefined \def \bctitle#1{#1}\fi
\ifx \beditor  \undefined \def \beditor#1{#1}\fi
\ifx \bpublisher  \undefined \def \bpublisher#1{#1}\fi
\ifx \bbtitle  \undefined \def \bbtitle#1{#1}\fi
\ifx \bedition  \undefined \def \bedition#1{#1}\fi
\ifx \bseriesno  \undefined \def \bseriesno#1{#1}\fi
\ifx \blocation  \undefined \def \blocation#1{#1}\fi
\ifx \bsertitle  \undefined \def \bsertitle#1{#1}\fi
\ifx \bsnm \undefined \def \bsnm#1{#1}\fi
\ifx \bsuffix \undefined \def \bsuffix#1{#1}\fi
\ifx \bparticle \undefined \def \bparticle#1{#1}\fi
\ifx \barticle \undefined \def \barticle#1{#1}\fi
\bibcommenthead
\ifx \bconfdate \undefined \def \bconfdate #1{#1}\fi
\ifx \botherref \undefined \def \botherref #1{#1}\fi
\ifx \url \undefined \def \url#1{\textsf{#1}}\fi
\ifx \bchapter \undefined \def \bchapter#1{#1}\fi
\ifx \bbook \undefined \def \bbook#1{#1}\fi
\ifx \bcomment \undefined \def \bcomment#1{#1}\fi
\ifx \oauthor \undefined \def \oauthor#1{#1}\fi
\ifx \citeauthoryear \undefined \def \citeauthoryear#1{#1}\fi
\ifx \endbibitem  \undefined \def \endbibitem {}\fi
\ifx \bconflocation  \undefined \def \bconflocation#1{#1}\fi
\ifx \arxivurl  \undefined \def \arxivurl#1{\textsf{#1}}\fi
\csname PreBibitemsHook\endcsname

\bibitem[\protect\citeauthoryear{Walkins}{1989}]{Walkins_1989}
\begin{botherref}
\oauthor{\bsnm{Walkins}, \binits{C.J.}}:
Learning from delayed rewards.
PhD thesis,
Cambridge University
(1989)
\end{botherref}
\endbibitem

\bibitem[\protect\citeauthoryear{Jin et~al.}{2020}]{Jin-2019}
\begin{bchapter}
\bauthor{\bsnm{Jin}, \binits{C.}},
\bauthor{\bsnm{Yang}, \binits{Z.}},
\bauthor{\bsnm{Wang}, \binits{Z.}},
\bauthor{\bsnm{Jordan}, \binits{M.I.}}:
\bctitle{Provably efficient reinforcement learning with linear function
  approximation}.
In: \beditor{\bsnm{Abernethy}, \binits{J.}},
\beditor{\bsnm{Agarwal}, \binits{S.}} (eds.)
\bbtitle{Proceedings of Thirty Third Conference on Learning Theory}.
\bsertitle{Proceedings of Machine Learning Research},
vol. \bseriesno{125},
pp. \bfpage{2137}--\blpage{2143}
(\byear{2020})
\end{bchapter}
\endbibitem

\bibitem[\protect\citeauthoryear{Ouyang et~al.}{2017a}]{ouyang-2017}
\begin{botherref}
\oauthor{\bsnm{Ouyang}, \binits{Y.}},
\oauthor{\bsnm{Gagrani}, \binits{M.}},
\oauthor{\bsnm{Nayyar}, \binits{A.}},
\oauthor{\bsnm{Jain}, \binits{R.}}:
Learning unknown {M}arkov decision processes: A {T}hompson sampling approach.
arXiv preprint arXiv:1709.04570
(2017)
\end{botherref}
\endbibitem

\bibitem[\protect\citeauthoryear{Ouyang et~al.}{2017b}]{TSDE_2017}
\begin{bchapter}
\bauthor{\bsnm{Ouyang}, \binits{Y.}},
\bauthor{\bsnm{Gagrani}, \binits{M.}},
\bauthor{\bsnm{Nayyar}, \binits{A.}},
\bauthor{\bsnm{Jain}, \binits{R.}}:
\bctitle{Learning unknown markov decision processes: A thompson sampling
  approach}.
In: \bbtitle{31st Conference on Neural Information Processing Systems},
\bconflocation{Long Beach, CA, USA}
(\byear{2017})
\end{bchapter}
\endbibitem

\bibitem[\protect\citeauthoryear{Jaksch et~al.}{2010}]{jaksch-2010}
\begin{barticle}
\bauthor{\bsnm{Jaksch}, \binits{T.}},
\bauthor{\bsnm{Ortner}, \binits{R.}},
\bauthor{\bsnm{Auer}, \binits{P.}}:
\batitle{Near-optimal regret bounds for reinforcement learning.}
\bjtitle{Journal of Machine Learning Research}
\bvolume{11}(\bissue{4}),
\bfpage{1563}--\blpage{1600}
(\byear{2010})
\end{barticle}
\endbibitem

\bibitem[\protect\citeauthoryear{Fruit et~al.}{2018}]{Fruit-2018}
\begin{botherref}
\oauthor{\bsnm{Fruit}, \binits{R.}},
\oauthor{\bsnm{Pirotta}, \binits{M.}},
\oauthor{\bsnm{Lazaric}, \binits{A.}},
\oauthor{\bsnm{Ortner}, \binits{R.}}:
Efficient bias-span-constrained exploration-exploitation in reinforcement
  learning
(2018)
\end{botherref}
\endbibitem

\bibitem[\protect\citeauthoryear{Tossou et~al.}{}]{Tossou-2019}
\begin{botherref}
\oauthor{\bsnm{Tossou}, \binits{A.}},
\oauthor{\bsnm{Basu}, \binits{D.}},
\oauthor{\bsnm{Dimitrakakis}, \binits{C.}}:
Near-optimal Optimistic Reinforcement Learning using Empirical Bernstein
  Inequalities.
\doiurl{10.48550/ARXIV.1905.12425} .
\url{https://arxiv.org/abs/1905.12425}
\end{botherref}
\endbibitem

\bibitem[\protect\citeauthoryear{}{}]{wei_model-free_2020}
\begin{botherref}
Model-free {Reinforcement} {Learning} in {Infinite}-horizon {Average}-reward
  {Markov} {Decision} {Processes}.
In: Proceedings of the 37th {International} {Conference} on {Machine}
  {Learning}.
Proceedings of {Machine} {Learning} {Research},
vol. 119.
\url{https://proceedings.mlr.press/v119/wei20c.html}
\end{botherref}
\endbibitem

\bibitem[\protect\citeauthoryear{Jin et~al.}{}]{Jin-2020}
\begin{botherref}
\oauthor{\bsnm{Jin}, \binits{C.}},
\oauthor{\bsnm{Kakade}, \binits{S.}},
\oauthor{\bsnm{Krishnamurthy}, \binits{A.}},
\oauthor{\bsnm{Liu}, \binits{Q.}}:
Sample-efficient reinforcement learning of undercomplete pomdps.
In: Larochelle, H., Ranzato, M., Hadsell, R., Balcan, M.F., Lin, H. (eds.)
Advances in Neural Information Processing Systems,
pp. 18530--18539
\end{botherref}
\endbibitem

\bibitem[\protect\citeauthoryear{Azizzadenesheli et~al.}{2016}]{Lazaric16}
\begin{bchapter}
\bauthor{\bsnm{Azizzadenesheli}, \binits{K.}},
\bauthor{\bsnm{Lazaric}, \binits{A.}},
\bauthor{\bsnm{Anandkumar}, \binits{A.}}:
\bctitle{Reinforcement learning of {POMDPs} using spectral methods}.
In: \bbtitle{{COLT}}.
\bsertitle{{JMLR} Workshop and Conference Proceedings},
vol. \bseriesno{49},
pp. \bfpage{193}--\blpage{256}
(\byear{2016})
\end{bchapter}
\endbibitem

\bibitem[\protect\citeauthoryear{Guo et~al.}{}]{GuoDB16}
\begin{botherref}
\oauthor{\bsnm{Guo}, \binits{Z.D.}},
\oauthor{\bsnm{Doroudi}, \binits{S.}},
\oauthor{\bsnm{Brunskill}, \binits{E.}}:
A {PAC} {RL} algorithm for episodic pomdps.
In: {AISTATS}.
{JMLR} Workshop and Conference Proceedings,
vol. 51,
pp. 510--518
\end{botherref}
\endbibitem

\bibitem[\protect\citeauthoryear{Papadimitriou and Tsitsiklis}{1987}]{PapTsi}
\begin{botherref}
\oauthor{\bsnm{Papadimitriou}, \binits{C.H.}},
\oauthor{\bsnm{Tsitsiklis}, \binits{J.N.}}:
The complexity of markov decision processes
(1987)
\end{botherref}
\endbibitem

\bibitem[\protect\citeauthoryear{Vlassis et~al.}{2012}]{Nikos2012}
\begin{botherref}
\oauthor{\bsnm{Vlassis}, \binits{N.}},
\oauthor{\bsnm{Littman}, \binits{M.L.}},
\oauthor{\bsnm{Barber}, \binits{D.}}:
On the computational complexity of stochastic controller optimization in pomdps
(2012)
\end{botherref}
\endbibitem

\bibitem[\protect\citeauthoryear{Even-Dar et~al.}{2005}]{resets05}
\begin{bchapter}
\bauthor{\bsnm{Even-Dar}, \binits{E.}},
\bauthor{\bsnm{Kakade}, \binits{S.M.}},
\bauthor{\bsnm{Mansour}, \binits{Y.}}:
\bctitle{Reinforcement learning in {POMDPs} without resets}.
(\byear{2005})
\end{bchapter}
\endbibitem

\bibitem[\protect\citeauthoryear{Ross et~al.}{}]{NIPS2007_3b3dbaf6}
\begin{botherref}
\oauthor{\bsnm{Ross}, \binits{S.}},
\oauthor{\bsnm{Chaib-draa}, \binits{B.}},
\oauthor{\bsnm{Pineau}, \binits{J.}}:
Bayes-adaptive {POMDPs}.
In: Platt, J., Koller, D., Singer, Y., Roweis, S. (eds.)
Advances in Neural Information Processing Systems
\end{botherref}
\endbibitem

\bibitem[\protect\citeauthoryear{Poupart and
  Vlassis}{2008}]{Poupart2008ModelbasedBR}
\begin{bchapter}
\bauthor{\bsnm{Poupart}, \binits{P.}},
\bauthor{\bsnm{Vlassis}, \binits{N.A.}}:
\bctitle{Model-based bayesian reinforcement learning in partially observable
  domains}.
In: \bbtitle{International Symposium on Artificial Intelligence and
  Mathematics}
(\byear{2008})
\end{bchapter}
\endbibitem

\bibitem[\protect\citeauthoryear{Anselmi et~al.}{2022}]{anselmi-2022}
\begin{bchapter}
\bauthor{\bsnm{Anselmi}, \binits{J.}},
\bauthor{\bsnm{Gaujal}, \binits{B.}},
\bauthor{\bsnm{Rebuffi}, \binits{L.-S.}}:
\bctitle{{Reinforcement Learning in a Birth and Death Process: Breaking the
  Dependence on the State Space}}.
In: \bbtitle{{NeurIPS 2022 - 36th Conference on Neural Information Processing
  Systems}},
\bconflocation{New Orleans, United States}
(\byear{2022}).
\burl{https://hal.science/hal-03799394}
\end{bchapter}
\endbibitem

\bibitem[\protect\citeauthoryear{Borgs et~al.}{2014}]{Borgs-2014}
\begin{bchapter}
\bauthor{\bsnm{Borgs}, \binits{C.}},
\bauthor{\bsnm{Chayes}, \binits{J.}},
\bauthor{\bsnm{Doroudi}, \binits{S.}},
\bauthor{\bsnm{Harchol-Balter}, \binits{M.}},
\bauthor{\bsnm{Xu}, \binits{K.}}:
\bctitle{The optimal admission threshold in observable queues with state
  dependent pricing}.
In: \bbtitle{Probability in the Engineering and Informational Sciences 28},
pp. \bfpage{101}--\blpage{110}
(\byear{2014}).
\burl{https://www.microsoft.com/en-us/research/publication/optimal-admission-threshold-observable-queues-state-dependent-pricing/}
\end{bchapter}
\endbibitem

\bibitem[\protect\citeauthoryear{Xia}{2014}]{Xia}
\begin{botherref}
\oauthor{\bsnm{Xia}, \binits{L.}}:
Event-based optimization of admission control in open queueing networks
(2014)
\end{botherref}
\endbibitem

\bibitem[\protect\citeauthoryear{Chandy et~al.}{1975}]{Norton}
\begin{barticle}
\bauthor{\bsnm{Chandy}, \binits{K.M.}},
\bauthor{\bsnm{Herzog}, \binits{U.}},
\bauthor{\bsnm{Woo}, \binits{L.}}:
\batitle{Parametric analysis of queuing networks}.
\bjtitle{IBM Journal of Research and Development}
\bvolume{19}(\bissue{1}),
\bfpage{36}--\blpage{42}
(\byear{1975})
\doiurl{10.1147/rd.191.0036}
\end{barticle}
\endbibitem

\bibitem[\protect\citeauthoryear{Rolia and Sevcik}{1995}]{MOL}
\begin{barticle}
\bauthor{\bsnm{Rolia}, \binits{J.A.}},
\bauthor{\bsnm{Sevcik}, \binits{K.C.}}:
\batitle{The method of layers}.
\bjtitle{IEEE Transactions on Software Engineering}
\bvolume{21}(\bissue{8}),
\bfpage{689}--\blpage{700}
(\byear{1995})
\doiurl{10.1109/32.403785}
\end{barticle}
\endbibitem

\bibitem[\protect\citeauthoryear{Rolia et~al.}{2009}]{Ju09}
\begin{bchapter}
\bauthor{\bsnm{Rolia}, \binits{J.}},
\bauthor{\bsnm{Casale}, \binits{G.}},
\bauthor{\bsnm{Krishnamurthy}, \binits{D.}},
\bauthor{\bsnm{Dawson}, \binits{S.}},
\bauthor{\bsnm{Kraft}, \binits{S.}}:
\bctitle{Predictive modelling of sap erp applications: Challenges and
  solutions}.
(\byear{2009})
\end{bchapter}
\endbibitem

\bibitem[\protect\citeauthoryear{}{2022}]{KnativeAdmission}
\begin{botherref}
Configuring Concurrency in Knative.
{https://knative.dev/docs/serving/autoscaling/concurrency/}.
Online; accessed: 2023-01-30
(2022)
\end{botherref}
\endbibitem

\bibitem[\protect\citeauthoryear{Wang et~al.}{2022}]{Ju22}
\begin{botherref}
\oauthor{\bsnm{Wang}, \binits{R.}},
\oauthor{\bsnm{Casale}, \binits{G.}},
\oauthor{\bsnm{Filieri}, \binits{A.}}:
Estimating multiclass service demand distributions using markovian arrival
  processes
(2022)
\end{botherref}
\endbibitem

\bibitem[\protect\citeauthoryear{Krieger}{2008}]{Bolch}
\begin{barticle}
\bauthor{\bsnm{Krieger}, \binits{U.R.}}:
\batitle{Queueing networks and markov chains, 2nd edition by g. bolch, s.
  greiner, h. de meer, and k.s. trivedi}.
\bjtitle{IIE Transactions}
\bvolume{40}(\bissue{5}),
\bfpage{567}--\blpage{568}
(\byear{2008})
\end{barticle}
\endbibitem

\bibitem[\protect\citeauthoryear{Kameda}{1984}]{kameda}
\begin{barticle}
\bauthor{\bsnm{Kameda}, \binits{H.}}:
\batitle{A property of normalization constants for closed queueing networks}.
\bjtitle{IEEE Transactions on Software Engineering}
\bvolume{SE-10}(\bissue{6}),
\bfpage{856}--\blpage{857}
(\byear{1984})
\doiurl{10.1109/TSE.1984.5010314}
\end{barticle}
\endbibitem

\bibitem[\protect\citeauthoryear{Puterman}{2014}]{puterman-1994}
\begin{bbook}
\bauthor{\bsnm{Puterman}, \binits{M.L.}}:
\bbtitle{Markov Decision Processes: Discrete Stochastic Dynamic Programming},
(\byear{2014})
\end{bbook}
\endbibitem

\bibitem[\protect\citeauthoryear{Levin et~al.}{2008}]{Peres-2008}
\begin{bbook}
\bauthor{\bsnm{Levin}, \binits{D.A.}},
\bauthor{\bsnm{Peres}, \binits{Y.}},
\bauthor{\bsnm{Wilmer}, \binits{E.L.}}:
\bbtitle{Markov Chains and Mixing Times},
(\byear{2008})
\end{bbook}
\endbibitem

\bibitem[\protect\citeauthoryear{Dopper et~al.}{2006}]{dopper-2006}
\begin{bchapter}
\bauthor{\bsnm{Dopper}, \binits{J.G.}},
\bauthor{\bsnm{Gaujal}, \binits{B.}},
\bauthor{\bsnm{Vincent}, \binits{J.-M.}}:
\bctitle{Bounds for the coupling time in queueing networks perfect simulation}.
In: \beditor{\bsnm{Langville}, \binits{A.N.}},
\beditor{\bsnm{Stewart}, \binits{W.J.}} (eds.)
\bbtitle{MAM, 150th Anniversary of A.A. Markov},
\bconflocation{Charleston, SC}
(\byear{2006})
\end{bchapter}
\endbibitem

\bibitem[\protect\citeauthoryear{Anselmi and Gaujal}{2014}]{anselmiGaujal-2014}
\begin{barticle}
\bauthor{\bsnm{Anselmi}, \binits{J.}},
\bauthor{\bsnm{Gaujal}, \binits{B.}}:
\batitle{{Efficiency of simulation in monotone hyper-stable queueing
  networks}}.
\bjtitle{{Queueing Systems}}
\bvolume{76}(\bissue{1}),
\bfpage{51}--\blpage{72}
(\byear{2014})
\doiurl{10.1007/s11134-013-9357-7}
\end{barticle}
\endbibitem

\bibitem[\protect\citeauthoryear{Azar et~al.}{2017}]{UCBVI}
\begin{bchapter}
\bauthor{\bsnm{Azar}, \binits{M.G.}},
\bauthor{\bsnm{Osband}, \binits{I.}},
\bauthor{\bsnm{Munos}, \binits{R.}}:
\bctitle{Minimax regret bounds for reinforcement learning}.
In: \bbtitle{International Conference on Machine Learning},
pp. \bfpage{263}--\blpage{272}
(\byear{2017})
\end{bchapter}
\endbibitem

\bibitem[\protect\citeauthoryear{Urgaonkar et~al.}{2005}]{UR05}
\begin{botherref}
\oauthor{\bsnm{Urgaonkar}, \binits{B.}},
\oauthor{\bsnm{Pacifici}, \binits{G.}},
\oauthor{\bsnm{Shenoy}, \binits{P.}},
\oauthor{\bsnm{Spreitzer}, \binits{M.}},
\oauthor{\bsnm{Tantawi}, \binits{A.}}:
An analytical model for multi-tier internet services and its applications
(2005)
\end{botherref}
\endbibitem

\bibitem[\protect\citeauthoryear{}{2022}]{3tier}
\begin{botherref}
AWS Architecture Center.
{https://aws.amazon.com/architecture}.
Online; accessed: 2023-06-19
(2022)
\end{botherref}
\endbibitem

\bibitem[\protect\citeauthoryear{Kritzinger
  et~al.}{1982}]{norton_generalization}
\begin{barticle}
\bauthor{\bsnm{Kritzinger}, \binits{P.S.}},
\bauthor{\bsnm{{van Wyk}}, \binits{S.}},
\bauthor{\bsnm{Krzesinski}, \binits{A.E.}}:
\batitle{A generalisation of norton's theorem for multiclass queueing
  networks}.
\bjtitle{Performance Evaluation}
\bvolume{2}(\bissue{2}),
\bfpage{98}--\blpage{107}
(\byear{1982})
\doiurl{10.1016/0166-5316(82)90002-5}
\end{barticle}
\endbibitem

\bibitem[\protect\citeauthoryear{Bhandari et~al.}{2018}]{Russo}
\begin{bchapter}
\bauthor{\bsnm{Bhandari}, \binits{J.}},
\bauthor{\bsnm{Russo}, \binits{D.}},
\bauthor{\bsnm{Singal}, \binits{R.}}:
\bctitle{A finite time analysis of temporal difference learning with linear
  function approximation}.
In: \beditor{\bsnm{Bubeck}, \binits{S.}},
\beditor{\bsnm{Perchet}, \binits{V.}},
\beditor{\bsnm{Rigollet}, \binits{P.}} (eds.)
\bbtitle{Proceedings of the 31st Conference On Learning Theory}.
\bsertitle{Proceedings of Machine Learning Research},
vol. \bseriesno{75},
pp. \bfpage{1691}--\blpage{1692}
(\byear{2018})
\end{bchapter}
\endbibitem

\bibitem[\protect\citeauthoryear{Ipsen and Meyer}{1994}]{Ipsen-1994}
\begin{barticle}
\bauthor{\bsnm{Ipsen}, \binits{I.C.F.}},
\bauthor{\bsnm{Meyer}, \binits{C.D.}}:
\batitle{Uniform stability of markov chains}.
\bjtitle{SIAM Journal on Matrix Analysis and Applications}
\bvolume{15},
\bfpage{1061}--\blpage{1074}
(\byear{1994})
\end{barticle}
\endbibitem

\end{thebibliography}

\begin{appendices}

\section{Proof of Theorem \ref{th:main}}
\label{app:proof}

 \subsection{Terms for the ramping phases}\label{ssec:init}
We first briefly deal with the terms coming from the ramping phases $\Phi$ of the episodes, $R_{\rm ramp}$.
We have:
\begin{equation}
 R_{\rm ramp}= \sum_k \sum_{t=t_k}^{t_k+\tm -1}r(s_t,\tilde \pi_k(s,t))
 \leq K \tm  r_{\max}
\leq r_{\max}SA\tm^{2}\log_2\prt{\frac{8T}{SA\tm}},
\end{equation}
where in the last inequality we used Lemma~\ref{Lem:number of episodes}.
Assuming $\tm SA\geq 4$, and using $\log(2)\geq \frac{1}{2}$, we rewrite it:
\begin{equation}
 R_{\rm ramp}
\leq 2r_{\max}SA\tm^{2} \log(2T).\label{eq:Rinit}
\end{equation}
This term is therefore among the lower-order terms of the regret.

\subsection{Terms in the confidence bound}\label{ssec:outside}

We start with the terms coming from the case where the MDP is out of the confidence regions $\mathcal{M}_k$.
For each episode $k$, we define:
\begin{itemize}
 \item $ V_k^{(\module)}(s) $  the number of visits of state $s$ during episode $k$ in module $\module$.
 \item $N_{t}^{(\module)}(s)$ is the number of visits of state $s$ until time-step $t$ excluded, in module $\module$.
 \item $\mathcal{M}(t)$ the set of MDPs $\mathcal{M}_k$ such that $t_k\leq t<t_{k+1}$
\end{itemize}
  For the terms out of the confidence sets, we have:
\begin{align*}
R_{\rm out}
 &\leq r_{\max}\sum_\module\sum_{s}\sum_{k=1}^{K} V_k^{(\module)}(s) \mathds{1}_{M \notin \mathcal{M}_k} \\
 &\leq r_{\max}\sum_\module\sum_{s}\sum_{k=1}^{K} N_{t_k}^{(\module)}(s) \mathds{1}_{M \notin \mathcal{M}_k}  \text{ using the stopping criterion}\\
 &=r_{\max}\sum_{t=1}^T \sum_{s}\sum_{k=1}^{K} \mathds{1}_{t_k=t}N_t(s) \mathds{1}_{M \notin \mathcal{M}(t)}
 \leq r_{\max}\sum_{t=1}^T \sum_{s}N_t(s) \mathds{1}_{M \notin \mathcal{M}(t)}\\
 &= r_{\max}\sum_{t=1}^T \mathds{1}_{M \notin \mathcal{M}(t)}\sum_{s}N_t(s)
 \leq r_{\max}\sum_{t=1}^T t\mathds{1}_{M \notin \mathcal{M}(t)}.
\end{align*}
We now need Lemma~\ref{Lem:conf outside} to control the probability that the MDP fails to be within the confidence bounds $\mathbb P \croc{M \notin \mathcal{M}(t)}$. Taking the expectations and using Lemma~\ref{Lem:conf outside}, we obtain
 \begin{align}
 \mathbb E \brac{R_{\rm out}}&\leq r_{\max} \sum_{t=1}^T t\Pb\croc{M\notin \mathcal{M}(t)}
 \leq r_{\max}\sum_{t=1}^T \frac{S'+16CS'A}{2t^2}\leq r_{\max}(S'+16CS'A).
  \label{eq:Rout}
 \end{align}
 This term is constant in $T$ and therefore it does not significantly contribute to the regret.

\subsection{Split of  confidence bound}\label{ssec:split}
We assume that $M \in \mathcal{M}_k$ and to simplify the notations, we will omit the use of the indicator functions $\1_{M\in \mathcal{M}_k}$.
For each episode $k$ and module $\module$, let us define
\begin{itemize}
\item $R_{{\rm in},k}^{(\module)} := \sum_s R_k^{(\module)}$,
\item $\tilde{\pi}_k$ the optimistic policy,
 \item $\tilde{P}_k:=\prt{\tilde{p}_k(s'|s,\tilde{\pi}_k(s))}$ the transition matrix of policy  $\tilde{\pi}_k$ on the optimistic MDP $\tilde{M}_k$,
 \item $\mathbf{v}_k^{(\module)}:=\prt{V_k(s,\tilde{\pi}_k}$ the row vector of visit counts,
 \item $\mathbf{h}_k$ the bias vector of the Markov chain in the true MDP $M$ with policy $\tilde \pi_k$.
\end{itemize}
Now, we split the regret term $R_{\rm in}^{(\module)}$ into subterms that have different meaning.
Assuming $M \in \mathcal{M}_k$ and using Lemma~\ref{EVIgain} on the accuracy of EVI, we get:
\begin{align*}
 R_{{\rm in},k}^{(\module)} &= \sum_{s,a} V_k^{(\module)}(s,a)(g^*-r(s,a))\\
 &\leq \sum_{s,a} V_k^{(\module)}(s,a)(\tilde{g}_k-r(s,a)) + \varepsilon_k\sum_{s,a}V_k^{(\module)}(s,a)\\
 &= \sum_{s,a} V_k^{(\module)}(s,a)(\tilde{g}_k-\tilde{r}_k(s,a)) +
 \sum_{s,a} V_k^{(\module)}(s,a)(\tilde{r}_k(s,a) - r(s,a)+
 \varepsilon_k) .
\end{align*}
In the next few steps, we will focus on rewriting the first sum. With \eqref{eq:EVIdiff} and using the definition of the iterated values from EVI, we have for a given state $s$ and $a_s:=\tilde{\pi}_k(s)$:
$$ \left|\prt{\tilde{g}_k-\tilde{r}_k(s,a_s)} - \prt{\sum_{s'}\tilde{p}_k(s'|s,a_s) u^{(k)}_i(s') -u^{(k)}_i(s)} \right| \leq \varepsilon_k,$$
so that:
$$
R_{{\rm in},k}^{(\module)}\leq \mathbf{v}_k^{(\module)}\left(\tilde{\mathbf{P}}_k - \mathbf{I} \right) \mathbf{u}_i 
+ \sum_{s,a} V_k^{(\module)}(s,a)(\tilde{r}_k(s,a) - r(s,a))
+ \varepsilon_k\sum_{s,a}V_k^{(\module)}(s,a).
$$
Again, with $\mathbf{\tilde{h}}_k$ being the bias of the average optimal policy for the optimist MDP, define:
$$d_k(s):=\left(u^{(k)}_i(s)-\min_x u^{(k)}_i(x)\right) - \left( \mathbf{\tilde{h}}_k(s)-\min_x \mathbf{\tilde{h}}_k(x)\right).$$
Then for any $s$: $\left| d_k(s)\right| \leq \varepsilon_k.$

Notice that the unit vector is in the kernel of $\left(\tilde{\mathbf{P}}_k - \mathbf{I} \right)$. Therefore, in the first term, we can replace $ \mathbf{u}_i$ by any translation of it. We get:
$$
\mathbf{v}_k^{(\module)}\left(\tilde{\mathbf{P}}_k - \mathbf{I} \right) \mathbf{u}_i =\mathbf{v}_k^{(\module)}\left(\tilde{\mathbf{P}}_k - \mathbf{I} \right) \mathbf{\tilde{h}}_k + \mathbf{v}_k^{(\module)}\left(\tilde{\mathbf{P}}_k - \mathbf{I} \right) \mathbf{d}_k.
$$
so that, using the definition of $\varepsilon_k$, we have that overall:
\begin{multline*}
R_{\rm in}^{(\module)}\leq
 \underbrace{\sum_k\mathbf{v}_k^{(\module)}\left(\tilde{\mathbf{P}}_k - \mathbf{I} \right) \mathbf{\tilde{h}}_k}_{R_{\rm bias}^{(\module)}}
+ \underbrace{\sum_k\mathbf{v}_k^{(\module)}\left(\tilde{\mathbf{P}}_k - \mathbf{I} \right) \mathbf{d}_k+ 2\delta_{\max}\sum_k \sum_{s,a}\frac{V_k^{(\module)}(s,a)}{\sqrt{t_k}}}_{R_{\rm EVI}^{(\module)}} \\
+ \underbrace{\sum_k \sum_{s,a} V_k^{(\module)}(s,a)(\tilde{r}_k(s,a) - r(s,a))}_{R_{\rm rewards}^{(\module)}}
\end{multline*}

We can already further simplify the term related to EVI. Notice that:
\begin{align*}
\mathbf{v}_k^{(\module)}\left(\tilde{\mathbf{P}}_k - \mathbf{I} \right) \mathbf{d}_k&\leq \sum_s V_k\left(s,\tilde{\pi}_k(s)\right)\cdot \left\|\tilde{p}_k\left(\cdot|s,\tilde{\pi}_k(s)\right)-\mathds{1}_{s} \right\|_1\cdot \sup_{s'}|d_k(s')| \\
&\leq 2\varepsilon_k\sum_s V_k^{(\module)}\left(s,\tilde{\pi}_k(s)\right)
\leq 2\delta_{\max}\sum_{s,a} \frac{V_k^{(\module)}\left(s,a\right)}{\sqrt{t_k}} \\
&\leq 2\delta_{\max}\sum_{s,a}\frac{V_k^{(\module)}(s,a)}{\sqrt{\max{\{1,N_{t_k}^{(\module_k(s,a))}(s,a)\}}}},
\end{align*}
where in the last inequality we used  that $\max\{1,N_{t_k}^{(\module_k(s,a))}(s,a)\}\leq t_k \leq T$.
Thus,
for $T\geq \frac{e^2}{2AT}$
the regret term coming from the consequences and approximations of EVI satisfies
\begin{equation}
 \label{eq:term EVI 1}
R_{\rm EVI}^{(\module)}\leq \delta_{\max}2\sqrt{2\log(2AT)}\sum_k\sum_{s,a}\frac{V_k^{(\module)}(s,a)}{\sqrt{\max{\{1,N_{t_k}^{(\module_k(s,a))}(s,a)\}}}}.
\end{equation}

Let us now deal with the term $R_{\rm rewards}^{(\module)}$, as it will be bounded by a similar term as in equation~\eqref{eq:term EVI 1}. Indeed, as $M \in \mathcal{M}_k$, we may use that both the optimistic and true rewards are within the confidence region from equation~\ref{confR}, and use that $t_k<T$, so that:
\begin{equation}
 \label{eq:term reward}
 R_{\rm rewards}^{(\module)} \leq  \delta_{\max}2\sqrt{2\log(2AT)}\sum_k\sum_{s,a}\frac{V_k(s,a)}{\sqrt{\max{\{1,N_{t_k}^{(\module_k(s,a))}(s,a)\}}}}
\end{equation}

On the other hand, we can also split more precisely the term that depends on the bias. Define $P_k$ as the transition matrix of the optimistic policy $\tilde{\pi}_k$ in the true MDP $M$. We get
%
\begin{multline}
\label{eq:multiterms}
 R_{\rm in}^{(\module)}  \hspace{-0.1cm}\leq
 \underbrace{\sum_k\mathbf{v}_k^{(\module)}\hspace{-0.1cm}\left(\tilde{\mathbf{P}}_k - \mathbf{P}_k \right) \hspace{-0.07cm}\mathbf{h}_k}_{R_{\rm trans}^{(\module)}}
 + \underbrace{\sum_k\mathbf{v}_k^{(\module)}\hspace{-0.1cm}\left(\tilde{\mathbf{P}}_k - \mathbf{P}_k \right)\hspace{-0.15cm} \prt{\mathbf{\tilde{h}}_k-\mathbf{h}_k}}_{R_{\rm diff}^{(\module)}}
 +\underbrace{\sum_k\mathbf{v}_k^{(\module)}\hspace{-0.1cm}\left(\mathbf{P}_k - \mathbf{I} \right)\hspace{-0.07cm} \mathbf{\tilde{h}}_k}_{R_{\rm ep}^{(\module)}} \\
+\underbrace{\delta_{\max}4\sqrt{2\log(2AT)}\sum_k\sum_{s,a}\frac{V_k^{(\module)}(s,a)}{\sqrt{\max{\{1,N_{t_k}^{(\module_k(s,a))}(s,a)\}}}}}_{R_{\rm EVI}^{(\module)}+R_{\rm rewards}^{(\module)}}.
\end{multline}

Now that we split the regret into several terms, we still need to sum over the modules and analyze for each term its contribution to the regret. For instance, we can sum over the modules the terms depending on EVI and the reward differences to get:
\begin{equation}
 R_{\rm EVI}+R_{\rm rewards}=\delta_{\max}4\sqrt{2\log(2AT)}\sum_k\sum_{s,a}\frac{V_k(s,a)}{\sqrt{\max{\{1,N_{t_k}^{(\module_k(s,a))}(s,a)\}}}}.
 \label{eq:EVIrew}
\end{equation}
This term is related to the choice of the confidence bounds, and it will contribute to the main term of the regret. Regarding the other terms, $R_{\rm trans}^{(\module)}$ will also use the confidence bounds on the transition as well as our knowledge of the bias in the true MDP. $R_{\rm diff}^{(\module)}$ will be a lower order term in the regret, using the confidence bounds for both the comparisons between the transitions and the biases. Finally, $R_{\rm ep}^{(\module)}$ will be related to the count of episodes, so that it will also be a lower order term. The discussion for each of these terms will be spread over the next subsections.

\subsection{Bound on \texorpdfstring{$R_{\rm trans}^{(\module)}$}{Rtrans}}\label{ssec:Rtrans}

To bound $R_{\rm trans}^{(\module)}$, we can follow the computations from \cite{anselmi-2022}. We will use our knowledge of the bias $\mathbf{h}_k$ and the control on the transitions in the optimistic MDP to simplify the regret term.

Notice that for a fixed state $1\leq s \leq S-1$:
\begin{equation*}
 \sum_{s'}p\prt{s'|s,\tilde{\pi}_k(s)}h_k(s')=\sum_{s'} p\prt{s'|s,\tilde{\pi}_k(s)}\prt{h_k(s')-h_k(s)}+h_k(s).
\end{equation*}
The same is true for $\tilde{p}_k$, and knowing the MDP is a birth and death process:
\begin{align*}
 R_{\rm trans}^{(\module)}&=\sum_k\sum_s \sum_{s'}V_k^{(\module)}\left(s,\tilde{\pi}_k(s)\right)\cdot \left(\tilde{p}_k\left(s'|s,\tilde{\pi}_k(s)\right)-p\left(s'|s,\tilde{\pi}_k(s)\right) \right)\cdot h_k(s') \\
 &=\sum_k\sum_s \sum_{s'} V_k^{(\module)}\left(s,\tilde{\pi}_k(s)\right) \left(\tilde{p}_k\left(s'|s,\tilde{\pi}_k(s)\right)-p\left(s'|s,\tilde{\pi}_k(s)\right) \right)\cdot \prt{h_k(s')-h_k(s)}\\
 &\leq \sum_k\sum_s V_k^{(\module)}\hspace{-0.1cm}\left(s,\tilde{\pi}_k(s)\right)\hspace{-0.05cm} \left\|\tilde{p}_k\left(\cdot|s,\tilde{\pi}_k(s)\right)\hspace{-0.05cm}-\hspace{-0.05cm}p\left(\cdot|s,\tilde{\pi}_k(s)\right) \right\|_1 \hspace{-0.05cm}\max\hspace{-0.1cm}\croc{\hspace{-0.05cm}\Delta^{\tilde \pi_k}\hspace{-0.05cm}(s),\Delta^{\tilde \pi_k}\hspace{-0.05cm}(s+1)\hspace{-0.05cm}}\\
 &\leq 4\sqrt{2 \log\left( 2AT\right)} \sum_k\sum_{s,a} \frac{\Delta(s+1) V_k^{(\module)}(s,a)}{\sqrt{\max\{1,N_{t_k}^{(\module_k(s,a))}(s,a)\}}},
 \end{align*}
where $\Delta$ is the difference of bias in the last inequality, we used the bound on the variations of the bias from Proposition~\ref{pro:bias}, and that the optimistic MDP has transitions close to the true transitions with inequality \eqref{confP}. Notice that the final term looks similar to the term coming from EVI and rewards related computations \eqref{eq:EVIrew}. We will deal with these terms together in the next subsection, as they are both mainly contributing to the regret.

\subsection{Bound on the main term}\label{ssec:main}

In  the previous Section~\ref{ssec:Rtrans}, we have shown that:
\begin{equation*}
R_{\rm trans}^{(\module)} \leq 4\sqrt{2\log\left( 2AT\right)} \sum_{s,a} \frac{\Delta(s+1) V_k^{(\module)}(s,a)}{\sqrt{\max\{1,N_{t_k}^{(\module_k(s,a))}(s,a)\}}}.
\end{equation*}
 Summing over the modules $\module$, we get:
 \begin{equation}
R_{\rm trans} \leq 4\sqrt{2\log\left( 2AT\right)} \sum_{s,a} \frac{\Delta(s+1) V_k(s,a)}{\sqrt{\max\{1,N_{t_k}^{(\module_k(s,a))}(s,a)\}}}.
\label{eq:trans tempo}
 \end{equation}
 We now wish to control this term, $R_{\rm EVI}$ and $R_{\rm rewards}$ using our knowledge of the bias, rather than bounding it directly with the diameter $D$. We first sum over the episodes and take the expectation, so that with Lemma~\ref{Lem:sum integral}, and using that $N_{t_k}^{(\module_k(s,a))}(s,a)\geq \frac{1}{\tm }N_{t_k}(s,a)$ we had from equation \eqref{eq:module}, we get:
\begin{align*}
 \mathbb{E}\left[\sum_{s,a}\sum_k \frac{\sqrt{\tm }V_k(s,a)}{\sqrt{\max\{1,N_{t_k}(s,a)\}}}\right] &\leq 3 \mathbb{E}\left[\sum_{s,a} \sqrt{\tm N_T(s,a)}\right]\\
 &\leq 3 \sum_{s} \sqrt{\tm \mathbb{E}\left[N_T(s)\right]A} , \,\quad \text{ by Jensen's inequality.}
\end{align*}
Therefore:
\begin{equation}
 R_{\rm trans}\leq 12\sqrt{2A\tm \log\left( 2AT\right)}\sum_{s=0}^{S}\Delta(s+1)\sqrt{\E\brac{N_T(s)}}.
\end{equation}
This is one of the terms mainly contributing to the regret, the other one being, doing similar computations:
\begin{equation}
  R_{\rm EVI}+R_{\rm rewards}\leq 12\delta_{\max}\sqrt{2A\tm \log\left( 2AT\right)}\sum_{s\geq 0}\sqrt{\E\brac{N_T(s)}}
\end{equation}
Now, let $N_T^{\pi^{\max}}$ be the number of visits when the starting state is sampled randomly from the initial distribution $\sm^{\pi^{\max}}$ and the policy $\pi^{\max}$ is always chosen. By stochastic ordering, as $N_T(s)\leq_{st}N_T^{\pi^{\max}}$, we have $\E\brac{N_T(s)}\leq \E \brac{N_T^{\pi^{\max}}}=T \sm^{\pi^{\max}}(s)$. We can therefore rewrite the main contributing term to the regret as:
\begin{equation}
 12\sqrt{2A\tm T\log\left( 2A\tm T\right)}\sum_{s= 0}^{S}\prt{\Delta(s+1)+\delta_{\max}}\sqrt{\sm^{\pi^{\max}}(s)}.
\end{equation}
Replace in the equation the choice $\tm =5\log T/\log \mrate^{-1}$ and
recall that we had, from Proposition~\ref{pro:bias}, $\Delta(s):= 2\delta_{\max} \sm^{\pi^{\max}}(0)^{-1} \sum_{i=1}^s \frac{U}{\mu(i)}\leq 2\delta_{\max} \sm^{\pi^{\max}}(0)^{-1}\frac{U}{\mu(1)}s$.
%
%
%
%
Using Lemma~\ref{lem:measure}, since
\begin{align*}
 \sum_{s=0}^{S}\prt{\Delta(s+1)+\delta_{\max}}\sqrt{\sm^{\pi^{\max}}(s)}&\leq 3\delta_{\max} \sm^{\pi^{\max}}(0)^{-1}\frac{U}{\mu(1)}\sum_{s=0}^{S} (s+1)\sqrt{\sm^{\pi^{\max}}(s)}\\
 &\leq 3\delta_{\max} \sm^{\pi^{\max}}(0)^{-1/2}\frac{U}{\mu(1)} \sqrt{C_1} \sum_{s\geq0} s\prt{\frac{\lambda}{\mu(i_0)}}^{s/2}\\
 &\leq  3\delta_{\max} \sm^{\pi^{\max}}(0)^{-1/2}\frac{U}{\mu(1)} \sqrt{C_1} \frac{1}{\prt{1-\sqrt{\frac{\lambda}{\mu(i_0)}}}^2}\\
 &\leq 3\delta_{\max}\frac{U}{\mu(1)} C_1 \frac{1}{\prt{1-\sqrt{\frac{\lambda}{\mu(i_0)}}}^3},
\end{align*}
then, assuming $\tm\leq T$, the main term is upper bounded by:
\begin{equation}
 72\delta_{\max} \frac{U}{\mu(1)} C_1 {\prt{1-\sqrt{\frac{\lambda}{\mu(i_0)}}}^{-3}}\log\left(AT\right) \sqrt{5AT\log^{-1}(\mrate^{-1})}.
 \label{eq:main term}
\end{equation}


 \subsection{Bound on \texorpdfstring{$R_{\rm diff}^{(\module)}$}{Rdiff}}\label{ssec:Rdiff}

We now deal with the term involving the difference of bias $R_{\rm diff}^{(\module)}$, defined in equation~\ref{eq:multiterms}. The proof mainly follows the one from \cite{anselmi-2022}, with a final tweak to relate the visits from a module to the total number of visits. Notice that we cannot directly use the confidence regions to control the difference between $\mathbf{\tilde{h}}_k$ and $\mathbf{h}_k$, so that we will need Lemma~\ref{Lem:bias difference}, and we are interested in controlling $\|\mathbf{\tilde{h}}_k-\mathbf{h}_k\|_\infty$.

Fix the module $\module$ and the episode $k$, with policy $\tilde \pi_k$.
Choose a state minimizing $N_{t_k}^{(\module_k(s,\tilde \pi_k(s)))}(s,\tilde \pi_k(s))$, and call this state $x_k$, $a_k:=\tilde \pi_k(x_k)$ and $\module':=\module_k(x_k,a_k)$: for this state, the confidence bounds are at their worst, and $ \sqrt{\frac{\log\prt{2At_k}}{\max\{1,N_{t_k}^{(\module')}(x_k,a_k)\}}}$ is maximal for episode $k$. This means that controlling the number of visits of the worst state lets us control the number of visits for any state. As the true MDP is within the confidence bounds, with a triangle inequality we get:

$$\|\tilde P_k-P_k\|_\infty \leq 4\sqrt{\frac{2\log\prt{2At_k}}{\max\{1,N_{t_k}^{(\module')}(x_k,a_k)\}}}.$$
%
We now want to use Lemma~\ref{Lem:bias difference}. In our case, notice that in the true MDP we have $D\geq T_{hit}^{\tilde \pi_k}\geq1$ for $S$ large enough. Remark also that $D^{\tilde \pi_k}$ can be replaced by $D$ in the last inequality of the proof of \ref{Lem:bias difference}, as $\mbox{span}(h^{\tilde \pi_k})\leq D$ by construction of $\tilde \pi_k$ with EVI, following the same argument as in \cite[Equation (11)]{jaksch-2010}.

\begin{equation}
 \|\mathbf{\tilde{h}}_k-\mathbf{h}_k\|_\infty \leq 8r_{\max}D^2\sqrt{\frac{2\log\prt{2At_k}}{\max\{1,N_{t_k}^{(\module')}(x_k,a_k)\}}}.
 \label{eq:biasbias}
\end{equation}
Hence,
\begin{align*}
 R_{\rm diff}^{(\module)} &\leq \sum_s \sum_{s'}V_k^{(\module)}\left(s,\tilde{\pi}_k(s)\right)\cdot \left(\tilde{p}_k\left(s'|s,\tilde{\pi}_k(s)\right)-p\left(s'|s,\tilde{\pi}_k(s)\right) \right)\cdot (\tilde{h}_k(s')-h_k(s')) \\
 &\leq \sum_s V_k^{(\module)}\left(s,\tilde{\pi}_k(s)\right)\cdot \left\|\tilde{p}_k\left(\cdot|s,\tilde{\pi}_k(s)\right)-p\left(\cdot|s,\tilde{\pi}_k(s)\right) \right\|_1 \|\mathbf{\tilde{h}}_k-\mathbf{h}_k\|_\infty\\
 &\leq 32D^2r_{\max} \log\prt{2AT} \Sigma^{(\module)},
 \end{align*}
 where in the last inequality we have used \eqref{eq:biasbias}
 and defined
 \begin{equation*}
  \Sigma^{(\module)}:= \sum_{s,a}\sum_k\sum_{t=t_k}^{t_{k+1}-1}\frac{\1_{\croc{s_t,a_t=s,a}}\1_{\croc{t\in \module}}}{\sqrt{\max\{1,N_{t_k}^{(\module_k(s,a))}(s,a)\}}\sqrt{\max\{1,N_{t_k}^{(\module')}(x_k,a_k)\}}}.
 \end{equation*}
By the choice  of $x_k$, $N_{t_k}^{(\module')}(x_k,a_k)\leq N_{t_k}^{(\module_k(s,a))}(s,a)$ for any state-action pair $(s,a)$, so that we can compute the sum $\Sigma:=\sum_{\module}\Sigma^{(\module)}$, with $I_k:=t_{k+1}-t_k$ the length of episode $k$:

 \begin{equation*}
  \Sigma \leq \sum_\module\sum_{s,a}\sum_k\sum_{t=t_k}^{t_{k+1}-1}\frac{\1_{\croc{s_t,a_t=s,a}}\1_{\croc{t\in \module}}}{\max\{1,N_{t_k}^{(\module')}(x_k,a_k)\}}= \sum_k \frac{I_k}{\max\{1,N_{t_k}^{(\module')}(x_k,a_k)\}}.
 \end{equation*}

Now, define $Q_{\max}:=\left(\frac{10C_2 S'^2}{\sm^{\pi^{\max}}(S)}\right)^2 \log \left(\left(\frac{10C_2 S'^2}{\sm^{\pi^{\max}}(S)}\right)^4\right)$ where we defined the constant $C_2=\frac{(\lambda \rcost + \hcost)C_1}{\mu(1)\prt{1-\lambda/\mu(i_0)}}$, and $I(T):=\max\croc{Q_{\max},T^{1/4}}$. We split the sum depending on whether the episodes are shorter than $I(T)$ or not, and call $K_{\leq I}$ the number of such episodes. This yields:
$$\Sigma\leq K_{\leq I} I(T) + \sum_{k,I_k> I(T)}  \frac{I_k}{\max\{1,N_{t_k}^{(\module')}(x_k,a_k)\}}.$$
Using the stopping criterion for episodes, and that we have chosen the module $\module'$  in equation~\eqref{eq:module} to have the inequality $ V_k^{(\module')}(x_k,a_k) \geq \frac{1}{\tm }V_k(x_k,a_k)$:
$$\Sigma\leq K_{\leq I} I(T) + \sum_{k,I_k> I(T)}  \frac{\tm I_k}{\max\{1,V_k(x_k,a_k)\}}.$$
Now we can end the computations as in \cite{anselmi-2022}. Denote by $\mathcal{E}$ the event: $$\mathcal{E}=\croc{\forall k \text{ s.t } I_k>I(T), \;\frac{1}{\max\{1,V_k(x_k,a_k)\}}\leq  \frac{2}{\sm^{\pi^{\max}}(S)I_k}}.$$
By splitting the sum, using the above event, we get:
\begin{align*}
 \Sigma &\leq K_{\leq I} I(T) + \1_\mathcal{E}\sum_{k,I_k>I(T)}\frac{2\tm }{\sm^{\pi^{\max}}(S)}+\1_{\mathcal{\bar{E}}}\sum_{k,I_k> I(T)} \tm I_k \\
 &\leq K_{\leq I} I(T) + \1_\mathcal{E}\prt{K-K_{\leq I}}\frac{2\tm }{\sm^{\pi^{\max}}(S)}+\1_{\mathcal{\bar{E}}}\tm T.
\end{align*}
We use Corollary~\ref{Cor:worst count} to get $\Pb\prt{\mathcal{\bar{E}}}\leq \frac{1}{4T}$, so that when taking the expectation:
\begin{align*}
 \E\brac{\Sigma}&\leq \E\brac{K_{\leq I}} I(T) + \E\brac{\prt{K-K_{\leq I}}}\frac{2\tm }{\sm^{\pi^{\max}}(S)}+\frac{\tm }{4}.
 \end{align*}
 Now using Lemma~\ref{Lem:number of episodes}, $S'A\geq4$, $I(T) \geq \frac{2}{\sm^{\pi^{\max}}(S)}$ and that $\frac{1}{\log 2}+ \frac{1}{4}\leq 2$:
 \begin{equation*}
 \E\brac{\Sigma}\leq \E\brac{K} I(T)\tm +\frac{\tm }{4}\leq 2S'A\tm \log (2AT) I(T).
\end{equation*}
Therefore, we have that:
\begin{equation}
 \E\brac{R_{\rm diff}^{(\module)}}\leq 64r_{\max}S'AD^2\tm  I(T)\log^2\prt{2AT} .
 \label{eq:bias term}
\end{equation}

\subsection{Bound on \texorpdfstring{$R_{\rm ep}$}{Rep}} \label{ssec:Rep}
The last regret term we have to bound is related to the count of episodes.
$$R_{\rm ep}^{(\module)}=\sum_k \mathbf{v}_k^{(\module)} \left(\mathbf{P}_k - \mathbf{I} \right) \mathbf{\tilde{h}}_k.$$
We first want to sum over the modules to get the same kind of term as in \cite{jaksch-2010}, written as a martingale difference sequence, and then take the expectation.  Following that proof, we define $X_t:= ~\left( p(\cdot | s_t,a_t) - \mathbf{e}_{s_t}\right)\mathbf{\tilde{h}}_{k(t)}\mathds{1}_{M \in \mathcal{M}_{k(t)}}$, where $k(t)$ is the episode containing step $t$ and $\mathbf{e}_i$ the vector with $i$-th coordinate $1$ and $0$ for the other coordinates.
We obtain
\begin{align*}
 \sum_c\mathbf{v}_k^{(\module)} \left(\mathbf{P}_k - \mathbf{I} \right) \mathbf{\mathbf{\tilde{h}}_k}&\leq \mathbf{v}_k \left(\mathbf{P}_k - \mathbf{I} \right) \mathbf{\mathbf{\tilde{h}}_k}
 \leq \sum_{t=t_k}^{t_{k+1}-1} X_t + \mathbf{\tilde{h}}_k(s_{t_{k+1}}) -\mathbf{\tilde{h}}_k(s_{t_k})\\
 &\leq \sum_{t=t_k}^{t_{k+1}-1} X_t + Dr_{\max},
\end{align*}
and by summing over the episodes we get
\begin{equation*}
 \sum_\module R_{\rm ep}^{(\module)} \leq \sum_{t=1}^{T} X_t +KDr_{\max}.
\end{equation*}
Notice that $\mathbb{E}\left[X_t|s_1,a_1, \dots, s_t, a_t \right]=0$, so that when taking the expectations, only the term in the number of episodes remains.

On the other hand, using Lemma~\ref{Lem:number of episodes} on the number of episodes, when taking the expectation we obtain
\begin{equation*}
 \mathbb{E}\brac{\sum_\module R_{\rm ep}^{(\module)}} \leq S'A\tm \log_2\left(\frac{8T}{S'A\tm }\right)\cdot Dr_{\max}.
\end{equation*}
As for the computation of \eqref{eq:Rinit}, assuming $\tm S'A\geq 4$:
\begin{equation}
 \mathbb{E}\brac{R_{\rm ep}} \leq 2r_{\max}S'AD\tm \log (2AT).
 \label{eq:ep term}
\end{equation}

\subsection{Total sum}\label{ssec:total}
We remind that we showed in subsection \ref{ssec:main} that the main term of the regret is:
\begin{equation*}
  72\delta_{\max} \frac{U}{\mu(1)} C_1 {\prt{1-\sqrt{\frac{\lambda}{\mu(i_0)}}}^{-3}}\log\left(AT\right) \sqrt{5AT\log^{-1}(\mrate^{-1})},
\end{equation*}

and it remains now to compute the lower order term of the regret $R_{\rm LO}$. Using  \eqref{eq:Rinit}, \eqref{eq:Rout}, \eqref{eq:bias term} and \eqref{eq:ep term}, the lower order term of the regret is upper bounded by, omitting the $r_{\max}$ factor:
 \begin{equation*}
 64S'AD^2\tm  I(T)\log^2\prt{2AT}
 +2S'A\tm (D+\tm )\log (2AT)+(S'+16CS'A),
\end{equation*}
and for $T$ large enough so that $1+16C\leq \log^2(T)$ the upper bound is :
\begin{equation*}
 r_{\max}   69S'AD^2\tm^{2} I(T)\log^2\prt{2AT},
\end{equation*}
which concludes the proof of Theorem  \ref{th:main}.

\section{Lemmas on Extended Value Iteration}\label{sec:EVI}
We remind the fundamental properties of the Extended Value Iteration (EVI) algorithm, first described in \cite{jaksch-2010}, which is used to find the optimistic MDP $\tilde{M}_k$ and the policy $\tilde{\pi}_k$ for each episode $k$ given a confidence region $\mathcal{M}_k$. These properties are useful notably in the first splits of the regret terms in Section \ref{ssec:split}.
EVI iteratively computes values in the following way:
\begin{align}
\begin{cases}
u^{(k)}_0(s)&=0 \nonumber \\
u^{(k)}_{i+1}(s)&=\max_{a \in \mathcal{A}}\left\{r(s,a)+\max_{p(\cdot)\in \mathcal{P}(s,a)}\left\{ \sum_{s \in S}p(s')u^{(k)}_i(s')\right\} \right\},
\end{cases}
\end{align}
where $\mathcal{P}(s,a)$ is the set of probabilities from \eqref{confP}, and the iterations are stopped with respect to the following lemma \cite[Theorem~7]{jaksch-2010}.
\begin{Lem}
\label{EVIgain}
For episode $k$ and accuracy $\varepsilon_k:=\frac{\delta_{\max}}{\sqrt{t_k}}$,  denote by $i$ the last step of extended value iteration, stopped when:
\begin{equation}
 \max_s\{u^{(k)}_{i+1}(s)-u^{(k)}_i(s)\} - \min_s\{u^{(k)}_{i+1}(s)-u^{(k)}_i(s)\} < \varepsilon_k.
\end{equation}
The optimistic MDP $\tilde{M}_k$ and the optimistic policy $\tilde{\pi}_k$
 at the last step of EVI are so that the gain is $\varepsilon_k-$ close to the optimal gain:
\begin{equation}
 \tilde{g}_k:=\min_s g(\tilde{M}_k,\tilde{\pi}_k,s) \geq \max_{M' \in \mathcal{M}_k,\pi,s'} g(M',\pi,s') - \varepsilon_k.
\end{equation}
\end{Lem}
Moreover, from \cite[Theorem~8.5.6]{puterman-1994}:
\begin{equation}
\left| u^{(k)}_{i+1}(s) - u^{(k)}_i(s) - \tilde{g}_k \right| \leq \varepsilon_k,
\label{eq:EVIdiff}
\end{equation}
and as the optimal policy yields an aperiodic unichain Markov chain, we have that $\tilde{g}_k=g(\tilde{M}_k,\tilde{\pi}_k,s)$ for any $s$, so that we can define the bias:
\begin{equation}
 \tilde{h}_k(s_0)=\mathbb{E}_{s_0}\brac{\sum_{t=0}^\infty (\tilde{r}(s_t,a_t)-\tilde{g}_k)}.
\end{equation}
Rather than using the last value of EVI in the computations of the regret, we rely on the bias to show that the last value and the optimistic bias are nearly equal, up to a translation. By choosing iteration $i$ large enough, from \cite[Equation~8.2.5]{puterman-1994}, we can ensure that:
\begin{equation}
 \left|u^{(k)}_i(s)-(i-1) \tilde{g}_k-\tilde{h}_k(s) \right| < \frac{\varepsilon_k}{2},
\end{equation}
so that we can define the following difference
\begin{equation}
 d_k(s):=\left|u^{(k)}_i(s) -\min_s u^{(k)}_i(s) - \left(\tilde{h}_k(s)-\min_s \tilde{h}_k(s)\right) \right| < \varepsilon_k.
\end{equation}

\section{Classical lemmas for the regret computation in {\sc UCRL}-like algorithms}\label{sec:classical}

We introduce classic lemmas from \cite{jaksch-2010} that are needed for the regret computations.
The first lemma, proven in \cite[Appendix~C.3]{jaksch-2010}, is used to simplify the main regret terms \eqref{eq:EVIrew} and \eqref{eq:trans tempo}.
\begin{Lem}
 \label{Lem:sum integral}
For any fixed state action pair $(s,a)$, and time $T$, we have:
 $$\sum_{t=1}^{T}\frac{\mathds{1}_{\{s_t,a_t=s,a\}}}{\sqrt{\max\{1,N_{t}(s,a)\}}} \leq 3\sqrt{N_{T+1}(s,a)}.$$
\end{Lem}

The next lemma, proven in \cite[Appendix~C.2]{jaksch-2010}
is useful to bound the term from in \ref{ssec:init} and in equation~\eqref{eq:bias term}.
\begin{Lem}
 \label{Lem:number of episodes}
Denote by $K_t$ the number of episodes up to time $t$, and let $t>SA\tm $. It is bounded by:
 $$K_t\leq S'A\tm \log_2\prt{\frac{8t}{S'A\tm }}.$$
\end{Lem}

The next lemma is needed in the proof of Lemma~\ref{Lem:McDiarmid-like}. While it includes the diameter, this will only impact a lower-order term of the regret.
\begin{Lem}[Azuma-Hoeffding inequality]
\label{Lem:Azuma}
Let $X_1,X_2, \dots$ be a martingale difference sequence with $|X_i| \leq RD$ for all $i$ and some $R>0$. Then, for all $\varepsilon >0$ and $n \in \N$:
$$
\Pb\croc{\sum_{i=1}^n X_i \geq \varepsilon} \leq \exp\prt{-\frac{\varepsilon^2}{2nDR}}.
$$
\end{Lem}

\section{Probability of not being in the confidence region}\label{sec:out}

We compute the probability that the true MDP $M$ fails to be in the confidence set. This lemma controls the corresponding regret terms in Section~\ref{ssec:outside} when we consider the episodes~$k$ with~$M\notin \mathcal{M}_k$.

Let us first  prove the key Lemma~\ref{lem:Russo}.

\begin{Lem}
 \label{lem:Russo}
Let us consider the  original MDP under any policy $\pi$,  with stationary measure $\sm^\pi$.
There exists $C>0$, $\mrate\in(0,1)$ such that:
\begin{equation}
\max_{\pi \in \Pi}\sup_{x_0\in\mathcal{S}} \left\|\mathbb{P^\pi}_{x_0}(x_t=\cdot)-\sm^\pi \right\|_{TV}\leq C\mrate^t \quad \forall t>0.
\end{equation}
Let $t,t'>0$ such that $t'-\tm \geq t$, with $t'$ and $t'-\tm $ belonging to the same episode. Let $X$ be a function of the state of the original MDP until time $t$ and $Y$ function of the state of the original MDP from time $t'$. Let $\hat{Y}$  be a random variable following the same distribution as $Y$ independently from $X$. Let $f$ be a real-valued, bounded function. Then:
\begin{equation*}
 \left| \mathbb{E}\left[ f(X,Y)\right]- \mathbb{E}\left[ f(X,\hat{Y})\right]\right| \leq 4C \|f\|_\infty  \mrate^{\tm }.
\end{equation*}
\end{Lem}

 \begin{proof}
 The proof is essentially the same as in \cite[Lemma 9]{Russo}, but as states are sampled from the original MDP and not a single Markov chain, we cannot just assume that the starting distribution at time $0$ is a stationary distribution. Instead, we have to make sure that it is the case for each start of the episodes, hence the initial phase where $\tm $ samples of the aggregated MDP are discarded, so that the original MDP is close to its stationary distribution. Due to this ramping time, we can make sure that $t'$ and $t'-\tm $ belong to the same episodes and can therefore be related to the same stationary distribution.

 Let $t,t'>0$ such that $t'-\tm \geq t$, with $t'$ and $t'-\tm $ belonging to the same episode $k$. Now, $X$ is a function of the state of the original MDP until time $t$, so there are $t$ observed transitions but there might be many more that are hidden. In turn, $Y$ is a function of the state of the original MDP from time $t'$. Let $\hat{Y}$  be a random variable following the same distribution as $Y$ and independent of $X$. Note that there are at least $\tm $ observed or hidden transitions between $t$ and $t'$ on the original MDP.

 We also define the distribution $P:=\Pb\croc{X\in \cdot,Y\in \cdot}$ and the distribution $Q~:=~\Pb\croc{X\in \cdot}\otimes\Pb\croc{Y\in \cdot}$, and we define  the total variation information $I_{TV}(X,Y):=\sum_x \Pb\croc{X=x}\left\|\Pb\croc{Y=\cdot \mid X=x}-\Pb\croc{Y=y} \right\|_{TV}$.
To simplify, assume that $\|f\|_\infty \leq \frac{1}{2}$. By definition of the total variation distance, we first have that:
 \begin{equation*}
    \left|\E\brac{f(X,Y)} -\E\brac{f(X,\hat Y)}\right|\leq \left\| P-Q\right\|_{TV},
 \end{equation*}
Then, using the properties of the total variation information related to a Markov chain described in \cite{Russo}, we obtain
 \begin{align*}
  \left\| P-Q\right\|_{TV}&\leq I_{TV}(X,Y) \leq I_{TV}(x_t,x_{t'})\leq I_{TV}(x_{t'-\tm },x_{t'})\\
  &\leq \sum_x \Pb\croc{x_{t'-\tm }=x}\left\|\Pb\croc{x_{t'}=\cdot \mid x_{t'-\tm }=x}-\Pb\croc{x_{t}=\cdot} \right\|_{TV}
 \end{align*}
then using a triangle inequality:
\begin{multline*}
 \left\|\Pb\croc{x_{t'}=\cdot \mid x_{t'-\tm }=x}-\Pb\croc{x_{t}=\cdot} \right\|_{TV}\leq
\left\|\Pb\croc{x_{t'}=\cdot}-\sm^{\tilde \pi_k} \right\|_{TV} +\\
 \left\|\Pb\croc{x_{t'}=\cdot \mid x_{t'-\tm }=x}-\sm^{\tilde \pi_k} \right\|_{TV},
\end{multline*}
we get 
\begin{equation*}
 \left\| P-Q\right\|_{TV} \leq 2C   \mrate^{\tm },
\end{equation*}

where in the last inequality we used assumption \eqref{eq:mixing} twice, as $t'$ and $t'-\tm $ belong to the same episode, and therefore can be related to the same stationary measure $\sm^{\tilde \pi_k}$. To clarify, the exponent $\tm $ in the inequality is loose, as $\tm $ is the number of time-steps in the aggregated MDP, so there are at least as many time steps in the original MDP, and the mixing is confirmed.
\end{proof}

We can now give the lemma that actually shows that $M$ is likely to be in the confidence set of MDPs.
\begin{Lem}
 \label{Lem:conf outside}
 For $t>1$, the probability that the MDP $M$ is not within the set of plausible MDPs $\mathcal{M}(t)$ is bounded by:
 \begin{equation*}
  \Pb\croc{M \notin \mathcal{M}(t)} \leq \frac{S'}{2t^3}+\frac{8CS'A}{t^{3}}.
 \end{equation*}
 \end{Lem}
Compared to \cite[Lemma~17]{jaksch-2010}, we notice that the first term comes from the choice of the confidence bound adapted to the birth and death structure of the MDP, but the second one comes from the imperfect independence of the observations. To prove this inequality, we  will need Lemma~\ref{lem:Russo} to consider independent events again, and to be able to use concentration inequalities.

Let us now prove Lemma~\ref{Lem:conf outside}.
\begin{proof}
 Fix a state-action pair  $(s,a)$, $\module$ any module and $n$ the number of visits of this pair within the module before time $t$. We will first consider the confidence around the empirical transitions, and then the confidence around the rewards.
Let $\varepsilon_p ~=~\sqrt{\frac{2}{n} \log\prt{16 At^4}}\leq \sqrt{\frac{8}{n}\log\prt{2 At}}.$
Define the events:
\begin{equation}
 \label{eq:event_prob}
A_n=\prt{\|\hat{p}^{(\module)}(\cdot|s,a)-p(\cdot|s,a)\|_1 \geq \sqrt{\frac{8}{n}\log\prt{2 At}}}
\end{equation}
Here, we aim to control these events but the difficulty is that the observations from the state-action pairs are not independent. On the other hand, we notice that the observations within a fixed module are nearly independent, which is why we needed to introduce these modules in the first place.

Define $\hat p^\bot(\cdot | s,a)$ 
the empirical transition probabilities
from $n$ independent observations of the state-action pair $(s,a)$.
Define events that are copies of $A_n$
but with independent observations:
\begin{equation}
A_n^\bot=\prt{\|\hat{p}^\bot(\cdot|s,a)-p(\cdot|s,a)\|_1 \geq \sqrt{\frac{8}{n}\log\prt{2 At}}}.
\label{eq:event_prob_ind}
\end{equation}
Similarly, define $A_n^{\bot,k}$
events such that the first $n-k$ observations are the same as the ones for $A_n$
 and the next $k$ observations are independent, so that for example $A_n^{\bot,0}=A_n$ and $A_n^{\bot,n-1}=A_n^{\bot}$.
Then, applying  $n-1$ times Lemma~\ref{lem:Russo}:
\begin{align*}
 \left| \mathbb{P}\croc{A_n}-\mathbb{P}\croc{A_n^\bot}\right| &\leq \sum_{k=1}^{n-1} \left| \mathbb{P}\croc{A_n^{\bot,k-1}}-\mathbb{P}\croc{A_n^{\bot,k}}\right|
\leq 4C n\mrate^{\tm }\leq 4CT^{1-5}.
\end{align*}

We can therefore work on the events with independent observations.
Knowing that from each pair, there are at most $3$ transitions, a Weissman's inequality gives:
 \begin{equation*}
  \Pb\croc{ \|\hat{p}^{(\module)}(\cdot|s,a)-p(\cdot|s,a)\|_1 \geq \varepsilon_p} \leq 6\exp\prt{-\frac{n\varepsilon_p^2}{2}}
 \end{equation*}
and we get
$$
 \Pb\croc{ A_n^\bot} \leq \frac{3}{8 At^4},
$$
and within our  choice  of $\tm $, $$\mathbb P \croc{A_n} \leq \frac{3}{8 At^4}+\frac{4C}{t^{4}}.$$

We deal with the rewards in a similar manner. Define the events:
\begin{equation}
 \label{eq:event_rew}
 B_n := \prt{|\hat{r}^{(\module)}(s,a)-r(s,a)| \geq \delta_{\max}\sqrt{\frac{2}{n}\log\prt{2 At}}}.
\end{equation}
By definition of $\hat{r}^{(\module)}(s,a)= \rcost\hat{p}^{(\module)}(s+1|s,a)+\frac{\hcost}{U}(S-s)$ \ref{eq:estimates_r_p}, and using that $\rcost \leq \delta_{\max}$, we can write:
\begin{equation*}
 \Pb \croc{B_n} \leq \Pb\croc{|\hat{p}^{(\module)}(s+1|s,a)-p(s+1|s,a)| \geq \sqrt{\frac{2}{n}\log\prt{2 At}}}.
\end{equation*}
Once again, we consider $\hat{p}^{\bot}(s+1|s,a)$ the empirical transition probabilities from independent observations of $(s,a)$ to $s+1$, and we look to control the probability of the events $B_n^{\bot}$. With the independence, we may now use the following Hoeffding inequality on the Bernoulli random variable of parameter $p(s+1|s,a)$:
\begin{equation*}
 \Pb\croc{ |\hat{p}^{(\module)}(s+1|s,a)-p(s+1|s,a)| \geq \varepsilon_r} \leq 2\exp\prt{-2n\varepsilon_r^2},
\end{equation*}
where $\varepsilon_r=\sqrt{\frac{1}{2n} \log\prt{16 At^4}}\leq \sqrt{\frac{2}{n}\log\prt{2 At}}$.
We therefore get:
\begin{equation*}
 \Pb \croc{B_n^\bot} \leq \frac{1}{8At^4},
\end{equation*}
and with the previous choice of $\tm $, $$\mathbb P \croc{B_n} \leq \frac{1}{8 At^4}+\frac{4C}{t^{4}}.$$
Overall:
$$\mathbb P \croc{A_n \cup B_n} \leq \frac{1}{2 At^4}+\frac{8C}{t^{4}}. $$

Now, with a union bound for all values of $n=\max\{1,N_t^{(\module)}(s,a)\} \in \croc{0, 1, \cdots, \left\lceil\frac{t-1}{\tm }\right\rceil}$ and all $\tm$ possible modules, and also summing over all state-action pairs:
 $$ \Pb\croc{M \notin \mathcal{M}(t)} \leq  \frac{S'}{2t^3}+\frac{8CS'A}{t^{3}}$$
as desired.
 \end{proof}

\section{Lemmas specific to our regret computations}\label{sec:biascomp}
In this section, we prove generic properties on the difference of biases between two MDPs. This control on the difference is needed in subsection \ref{ssec:Rdiff} to compare the optimistic MDP and the true MDP.

 \subsection{Lemmas on the bias differences}

The next three lemmas of this subsection are already proved in \cite{anselmi-2022}, for the sake of completeness, we rewrite them in this appendix. They are used in the proof of Lemma~\ref{Lem:bias difference}, to control the difference between the bias of the policy $\tilde \pi_k$ in the optimistic MDP and in the true MDP.

\begin{Lem}
 \label{Lem:span}
  For an MDP
 with rewards $r \in [0,r_{\max}]$ and transition matrix $P$, denote by $J_s(\pi,T):=\E \brac{\sum_{t=0}^T r(s_t,\pi(s_t))}$ the expected cumulative rewards until time $T$ starting from state $s$, under policy $\pi$. Let  $D_\pi$ be the diameter under policy $\pi$.
 The following inequality holds: $\mbox{span }(J(\pi,T)) \leq r_{\max}D_\pi$.
\end{Lem}

\begin{proof}
 Let $s,s' \in \mathcal{S}$ be recurrent states under policy $\pi$. Call $\tau_{s\to s'}$ the random time needed to reach state $s'$ from state $s$. Then:
 \begin{align*}
  J_s(\pi,T) &= \E\brac{\sum_{t=0}^T r(s_t)}\\
  &= \E\brac{\sum_{t=0}^{\tau_{s\to s'}-1} r(s_t)}+\E\brac{\sum_{t=\tau_{s \to s'}}^T r(s_t)}\\
  &\leq r_{\max}\E\brac{\tau_{s \to s'}}+J_{s'}(\pi,T) \\
  &\leq r_{\max} D_{\pi} +J_{s'}(\pi,T),
 \end{align*}
which proves the lemma.
\end{proof}

\begin{Lem}
 \label{Lem:gain difference}
 Consider two unichain MDPs $M$ and $M'$. Let $r=r'\in [0,r_{\max}]$ and $P,P'$ be the rewards and transition matrix of  MDP $M,M'$ under policy $\pi,\pi'$ respectively, where both MDPs have the same state and action spaces. Denote by $g,g'$ the  average reward obtained under policy $\pi,\pi'$ in the MDP $M,M'$ respectively. Then  the difference of the gains is upper bounded.
 $$|g-g'| \leq  r_{\max}D_{\pi} \| P-P'\|_\infty.$$
\end{Lem}
\begin{proof}
Define for any state $s$ the following correction term  $b(s):=r_{\max}D_{\pi}\|p(\cdot|s)-p'(\cdot|s)\|_1.$ Let us show by induction that for $T \geq 0$,
$$ \sum_{t=0}^{T-1} P^t r \leq \sum_{t=0}^{T-1} P'^t (r+b) .$$
This is true for $T=0$. Assume that the inequality is true for some $T\geq0$, then
\begin{align*}
 \sum_{t=0}^{T} P^t r - \sum_{t=0}^{T} P'^t (r+b) &=-b + P\sum_{t=0}^{T-1} P^t r - P'\sum_{t=0}^{T-1} P'^t (r+b)\\
 &=-b + P'\prt{ \sum_{t=0}^{T-1} P^t r -\sum_{t=0}^{T-1} P'^t (r+b)} + (P-P')\sum_{t=0}^{T} P^t r\\
 &\leq -b + (P-P')\sum_{t=0}^{T} P^t r \text{ by induction hypothesis.}
\end{align*}
Notice that, for any recurrent state $s$ for policy $\pi$:
\begin{align*}
 \prt{(P-P')\sum_{t=0}^{T} P^t r }(s) & \leq \|p(\cdot|s)-p'(\cdot|s)\|_1 \cdot \mbox{span }(J(T))\\
 &\leq r_{\max}D_{\pi}\|p(\cdot|s)-p'(\cdot|s)\|_1 \text{ by Lemma~\ref{Lem:span}}\\
 &= b(s).
\end{align*}
In the same manner we show that:
$$ \sum_{t=0}^{T} P^t r \geq \sum_{t=0}^{T} P'^t (r-b).$$
Hence, as $P'$ has non-negative coefficients, denoting by $e$ the unit vector:
$$\left\|\sum_{t=0}^{T} P^t r - \sum_{t=0}^{T} P'^t r\right\|_\infty \leq \|b\|_\infty \left\|\sum_{t=0}^{T}P'^t \cdot e \right\|_\infty=\|b\|_\infty (T+1).$$

As $r=r'$, with a multiplication by $\frac{1}{T+1}$ and by taking the Ces\'aro limit :
$$|g-g'| \leq  \|b\|_\infty ,$$
where $\|b\|_\infty=r_{\max}D_{\pi} \| P-P'\|_\infty$.
\end{proof}

\begin{Lem}
 \label{Lem:hitting time}
  Let $P$ be the stochastic matrix of an ergodic Markov chain with state space $1,\dots, S$. The matrix $A:=I-P$ has a block decomposition
  $$A =\begin{pmatrix}
 A_S &  b \\
 c & d\end{pmatrix};
 $$
 then  $A_S$,  of  size $S \times S$  is invertible and $\| A_S^{-1}\|_{\infty}= \sup_{i\in\mathcal{S}} \E \, \tau_{i\to S}$, where $\E\,\tau_{i\to S}$ is the expected time to reach state $S$ from  state $i$.
\end{Lem}
Remark that this lemma is true for any state in $\mathcal{S}$.
\begin{proof}
$\prt{\E\, \tau_{i\to S}}_i$ is the unique vector solution to the system:
$$\begin{cases}
 v(S)=0 \\
 \forall i \neq S, \, v(i)=1+\sum_{j \in \mathcal{S}} P(i,j)v(j)
\end{cases}$$
We can rewrite this system of equations as: $\tilde{A}v=e-e_S$, where $\tilde{A}$ is the matrix $$\tilde{A}:=\begin{pmatrix}
 A_S &  b \\
 0 & 1\end{pmatrix},$$
 $e$ the unit vector and $e_S$ the vector with  value $1$ for the last state and $0$ otherwise.
 Then $\widetilde{A}$ and $A_S$ are invertible and we write:
 $$\tilde{A}^{-1}=
 \begin{pmatrix}
 A_S^{-1} &  -A_S^{-1}b \\
 0 & 1
 \end{pmatrix}.$$
 Thus, by computing $\tilde{A}^{-1}(e-e_S)$, for $i\neq S$, $\prt{\E\,\tau_{i\to S}}_i= A_S^{-1} e$. By definition of the infinite norm and using that  $A_S$ is an M-matrix and that its inverse has non-negative components, $\| A_S^{-1}\|_{\infty}= \sup_{i\in\mathcal{S}} \E \,\tau_{i\to S}$.
\end{proof}

In the following lemma, we use the same notations as in  Lemma~\ref{Lem:gain difference} with a common state space $\{ 0,1,\ldots S\}$.
\begin{Lem}
 \label{Lem:bias difference}
Let the biases $h$, $h'$ be the biases of the two MDPs that verify their respective Bellman equations with the renormalization choice $h(S)=h'(S)=0$, and respective policies $\pi,$ and $\pi'$. Let $\sup_{s \in \mathcal{S}} \E\, \tau^\pi_{s \to s'}$  be the  worst expected hitting time to reach the state $s'$ with policy $\pi$, and call $T_{hit}:=\inf_{s' \in \mathcal{S}} \sup_{s \in \mathcal{S}} \E\,\tau_{s \to s'}^\pi$.
 We have the following control of the difference:
$$
 \|h-h'\|_\infty \leq 2 T_{hit}^\pi D^{\pi'}r_{\max} \|P-P'\|_\infty.
 $$
\end{Lem}
Notice that although the biases are unique up to a constant additive term, the renormalization choice does not matter as the unit vector is in the kernel of $(P-P')$.
\begin{proof}
The computations in this proof follow the same idea as in the proof of  \cite[Theorem 4.2]{Ipsen-1994}.
 The biases verify the following Bellman equations $r -g e = (I-P)h$, and also the arbitrary renormalization equations, thanks to the previous remark: $h(S)=0$. Using the same notations as in the proof of Lemma~\ref{Lem:hitting time}, we can write the system of equations $\tilde{A}h=\tilde{r}-\tilde{g}$, with
 $\tilde{r}$ and $\tilde{g}$ respectively equal to $r$ and $g$ everywhere but on the last state, where their value is replaced by $0$.

 We therefore have that $h=\tilde{A}^{-1}(\tilde{r}-\tilde{g})$, and with identical computations, $h'=\tilde{A'}^{-1}(\tilde{r'}-\tilde{g'}).$ By denoting $dX:=X-X'$ for any vector or matrix $X$, we get, as $r=r'$:
 $$dh=-\tilde{A}^{-1}(-d\tilde{g}+d\tilde{A}h').$$
 The previously  defined block  decompositions are:
 $$\tilde{A}^{-1}=
 \begin{pmatrix}
 A_S^{-1} &  -A_S^{-1} b \\
 0 & 1
 \end{pmatrix}%
 \quad\text{ and }\quad
 d\tilde{A}=
 \begin{pmatrix}
 A_S-A_S' &  b-b' \\
 0 & 0
 \end{pmatrix}.$$
 For $s<S$, $dh(s)=-e_s^T A_S^{-1}(dA_S h'-d\tilde{g})$ and $dh(S)=0$.
 Now by taking the norm and using \ref{Lem:span}:
 $$\|dh\|_\infty\leq \|A_S^{-1}\|_\infty(r_{\max}D^{\pi'}\|dA_S\|_\infty+|d\tilde{g}|).$$
 Notice that $\|dA_S\|_\infty\leq\|dP\|_\infty$ and $|d\tilde{g}|=|dg| $. Using Lemma~\ref{Lem:gain difference} and Lemma~\ref{Lem:hitting time}, and taking the infimum for the choice of the state of renormalization implies the claimed inequality for the biases.
\end{proof}

 \subsection{Visits of the furthest state}
 We also need the next lemmas to bound $R_{\rm diff}$ by controlling the number of visits of the state with the fewest visits. If we can guarantee that each state receives enough visits, then we will have a good approximation of the biases and transition probabilities. The proof can be found in \cite{anselmi-2022}.
 \begin{Lem}
\label{Lem:McDiarmid-like}
Let $\sm^{\pi^{\max}}$ be the stationary measure of the Markov chain under policy $\pi^{\max}$, such that for every state $s$: $\pi^{\max}(s)=1$, so that every job is admitted in the network until maximal capacity $S$ is reached.

 Let $k$ be an episode and assume that the length of this episode $I_k$ is at least $I(T)=1+\max\croc{Q_{\max},T^{1/4}}$, with $Q_{\max} := \left(\frac{10C_2 S'^2}{\sm^{\pi^{\max}}(S)}\right)^2\log \left(\left(\frac{10C_2 S'^2}{\sm^{\pi^{\max}}(S)}\right)^4\right)$, $C_2:=\frac{(\lambda \rcost + \hcost)C_1}{\mu(1)\prt{1-\lambda/\mu(i_0)}}$ and $C_1$ as in Lemma~\ref{lem:measure}. Then, with probability at least $1-\frac{1}{4T}$:
 \begin{equation*}
  V_k(x_k,a_k)\geq \sm^{\pi^{\max}}(S)I_k-5C_2 S'^2\sqrt{I_k\log I_k}.
 \end{equation*}
\end{Lem}

We will now prove Lemma~\ref{Lem:McDiarmid-like}:

\begin{proof}
Let $k$ be an episode such that $I_k\geq I(T)$, and first consider it is of fixed length $I$. Let $x_k\in \mathcal S$ be a recurrent state, $a_k=\tilde \pi_k(s_k)$. Denote by $\sm_k$ the stationary distribution under policy $\tilde \pi_k$. Notice that $\sm^{\pi^{\max}}(S)\leq \sm_k(x_k)$ for $S$ large enough.

Define a new Markov reward process: consider again the original state space $\mathcal S'$ and the transitions $p'$ with policy $\tilde \pi_k$, but the rewards $\mathring{r}$, where $\mathring{r}(s')=1$ for states $s'$ such that $|s'|=x_k$ and $0$ otherwise. Denote by $\mathring{g}_{\tilde \pi_k}$ the gain associated to the policy $\tilde \pi_k$ and similarly define $\mathbf{\mathring{h}}_{\tilde \pi_k}$ the bias, translated so that $\mathring{h}_{\tilde \pi_k}(S)=0$.
Then:
 \begin{align*}
V_k(x_k,a_k)&=\sum_{u=t_k}^{t_{k+1}-1} \mathring{r}(s'_u)\\
&=\sum_{u=t_k}^{t_{k+1}-1} \mathring{g}_{\tilde \pi_k}+\mathring{h}_{\tilde \pi_k}(s'_u)-\left\langle p'\prt{\cdot| s'_u,\tilde \pi_k(s'_u)},\mathring{h}_{\tilde \pi_k}\right\rangle \text{ using a Bellman equation}\\
&=\sum_{u=t_k}^{t_{k+1}-1} \mathring{g}_{\tilde \pi_k}+\mathring{h}_{\tilde \pi_k}(s'_u)-\mathring{h}_{\tilde \pi_k}(s'_{u+1})+\mathring{h}_{\tilde\pi_k}(s'_{u+1})-\left\langle p'\prt{\cdot| s'_u,\tilde \pi_k(s'_u)},\mathring{h}_{\tilde \pi_k}\right\rangle.
 \end{align*}

By Azuma-Hoeffding inequality \ref{Lem:Azuma}, following the same proof as in section 4.3.2 of \cite{jaksch-2010}, notice that $X_u=\mathring{h}_{\tilde \pi_k}(s'_{u+1})-\left\langle p'\prt{\cdot \mid s'_u,\tilde \pi_k(s'_u)},\mathring{h}_{\tilde \pi_k}\right\rangle$ form a martingale difference sequence with the bound $|X_u|\leq \mbox{span}\,\mathring{h}_{\tilde \pi_k} $:
$$
\Pb\croc{\sum_{u=t_k}^{t_{k+1}-1} X_u \geq C_2 S'^2\sqrt{10I\log I}} \leq \frac{1}{I^5}.
$$
With Proposition~\ref{pro:bias} proved in Appendix~\ref{sec:aggregated}, we have $\mbox{span}\,\mathring{h}_{\tilde \pi_k}\leq C_2 S'^2$ with $C_2=\frac{(\lambda \rcost + \hcost)C_1}{\mu(1)\prt{1-\lambda/\mu(i_0)}}$, so that
with probability at least $1-\frac{1}{I^2}$:
\begin{equation*}
 V_k(x_k,a_k)\geq \sum_{u=t_k}^{t_{k+1}-1} \mathring{g}_{\tilde \pi_k} - 5C_2 S'^2\sqrt{I\log I}.
\end{equation*}
On the other hand:
\begin{equation*}
 \sum_{u=t_k}^{t_{k+1}-1} \mathring{g}_{\tilde \pi_k} = V_k(s_k,a_k) \sm_k(x_k),
\end{equation*}
so that, using that $\sm_k(x_k) \geq \sm^{\pi^{\max}}(S)$, with probability at least $1-\frac{1}{I^5}$:
\begin{equation*}
 V_k(x_k,a_k)\geq \sm^{\pi^{\max}}(S)I - 5C_2 S'^2\sqrt{I\log I}.
\end{equation*}
We now use a union bound over the possible values of the episode lengths $I_k$, between $I(T)+1$ and $T$:
\begin{align*}
 \Pb\croc{V_k(x_k,a_k)< \sm^{\pi^{\max}}(S)I_k - 5C_2 S'^2\sqrt{I_k\log I_k}}&\leq \sum_{I=I(T)+1}^T \frac{1}{I^5}\leq \sum_{I=T^{1/4}+1}^T \frac{1}{I^5} \\
 &\leq \frac{1}{4T},
\end{align*}
so that we now have that with probability at least $1-\frac{1}{4T}$:
\begin{equation*}
 V_k(x_k,a_k)\geq \sm^{\pi^{\max}}(S)I_k - 5C_2 S'^2\sqrt{I_k\log I_k}.
\end{equation*}
\end{proof}

We can show a corollary of Lemma~\ref{Lem:McDiarmid-like} that we will use for the regret computations:
\begin{Cor}
 \label{Cor:worst count}
For an episode $k$ such that its length $I_k$ is greater than $I(T)$,with probability at least $1-\frac{1}{4T}$:
 \begin{equation*}
    V_k(x_k,a_k)\geq \frac{\sm^{\pi^{\max}}(S)}{2}I_k.
 \end{equation*}
\end{Cor}

\begin{proof}
 With Lemma~\ref{Lem:McDiarmid-like}, it is enough to show that $5C_2 S'^2\sqrt{I_k\log I_k} \leq \frac{\sm^{\pi^{\max}}(S)}{2}I_k$, \emph{i.e.} that $ \sqrt{\frac{I_k}{\log I_k}} \geq\frac{10C_2 S'^2}{\sm^{\pi^{\max}}(S)}=:B$. By monotonicity, as $I_k\geq Q_{\max}={B}^2\log {B}^4$ we can show instead that ${B}^2 \log {B}^4 \geq {B}^2 \log \prt{ {B}^2\log {B}^4}$.

 This last inequality is true, using that $\log x \geq \log(2 \log x)$ for $x > 1$. This proves the corollary.
\end{proof}

\section{Properties of the aggregated MDP} \label{sec:aggregated}

In this section, we prove properties on the aggregated MDP that are needed to control the average number of visits of the states of the MDP under any policy. We also prove a bound on the bias of the true MDP under any policy, which is eventually needed to control the main term in subsections \ref{ssec:Rtrans} and \ref{ssec:main}.

\subsection{Properties of the policies in the aggregated MDP}
 We may only consider policies that are threshold policies, as we are mainly interested in the average reward scored by these policies, so that we consider that the policies chosen by EVI are threshold policies.
 We remind that the aggregated MDP is stable (as seen in Section \ref{sec:network}), so that there exists a $i_0$ large enough for which $i\geq i_0$, $\mu(i)\geq\mu(i_0)>\lambda$.

 With the following lemma, we compute the stationary measures $\sm^{\pi}$ and give a comparison between any $\sm^\pi$ with the stationary measure $\sm^{\pi^{\max}}$ of the maximal policy $\pi^{\max}$, that admits every job into the queue, by relating these Markov chains to the $M/M/1/S$ queue with rates $\lambda$ and $\mu(i_0)$.
\begin{Lem}
\label{lem:measure}
Denote by $\bar s$ the last recurrent state of the MDP for policy $\pi$, so that $\pi(s)=0$ for $s\geq \bar s$. Define the constant
$C_1:=\prod_{i=1}^{i_0-1}\frac{\mu(i_0)}{\mu(i)}\geq1$, independent of $S$.

 We have the following inequalities
 \begin{itemize}
  \item On the stationary measure of the maximal policy:
  \begin{equation*}
         \sm^{\pi^{\max}}(0)^{-1}:=\sum_{s'= 0}^{S}\prod_{i=1}^{s'}\frac{\lambda}{\mu(i)}\leq \frac{C_1  }{1-\frac{\lambda}{\mu(i_0)}},
        \end{equation*}
\item On the stationary measure of any policy:
\begin{equation*}
  \sm^{\pi}(0)^{-1}:=\sum_{s'=0}^{\bar s}\prod_{i=1}^{s'}\frac{\lambda}{\mu(i)}\leq  \sm^{\pi^{\max}}(0)^{-1} \leq \frac{C_1  }{1-\frac{\lambda}{\mu(i_0)}},
 \end{equation*}
\item Also we can compute for $s\leq S$:
 \begin{equation*}
  \sm^{\pi^{\max}}(s):= \sm^{\pi^{\max}}(0)\prod_{i=1}^{s}\frac{\lambda}{\mu(i)}= \sm^{\pi^{\max}}(0)C_1\prt{\frac{\lambda}{\mu(i_0)}}^{s}.
 \end{equation*}
 \end{itemize}
\end{Lem}

 We now remind a definition of the bias for any policy $\pi$ and control its variations, as they play a major role in the computations of the main term of the regret (see \ref{ssec:Rtrans}).
 \begin{Def}[Bias]
  Let $\pi$ be a policy, $P$ the transition matrix and $\sm^{\pi}$ the stationary measure of the Markov chain under policy $\pi$. The bias $\mathbf h^{\pi}$ of this policy is defined as:
 \begin{equation}
   h^{\pi}(s)=\sum_{t=1}^{\infty} (P^t(s,\cdot)-\sm^\pi)\mathbf{r}.
 \end{equation}
 \end{Def}
In order to control the variation of the bias of any policy, we will relate the bias to the expected hitting time to hit the state $0$ from state $s$, so that we first need to compute the hitting times:

\begin{Lem}
\label{Lem:hit0}
 Let $\pi$ be any policy, $(X_t)_t$ be the Markov chain with policy $\pi$ and transitions $P$ starting from any state $s$. Denote by $\tau_{s}$ the random time needed for $X_t$ to hit $0$. Then:
 \begin{equation*}
   \mathbb{E}\tau_{s}\leq \sm^\pi(0)^{-1} \sum_{i=1}^s \frac{U}{\mu(i)}
 \end{equation*}
\end{Lem}
\begin{proof}
 We write the expected hitting time equations, and use induction. Let $\tau_i$ be the hitting time to $0$ starting from state $i$, and $\mathbf e$ be the unit vector. We have the system:

\begin{equation}
 \E \boldsymbol{\tau}= \mathbf{e} + P\; \E\boldsymbol{\tau},
\end{equation}
with the extra equation $\tau_0=0$. The system gives for $s \geq \bar s$:
\begin{equation*}
 \mathbb{E}\tau_{s} = 1 + \mathbb{E}\tau_{s}\frac{1-\mu(s)}{U}+\mathbb{E}\tau_{s-1}\frac{\mu(s)}{U},
\end{equation*}
so that:
\begin{equation*}
 \mathbb{E}\tau_s = \frac{U}{\mu(s)}+\mathbb{E}\tau_{s-1}
\end{equation*}
Then, by induction, we want to prove the equation for $s<\bar s$:
\vspace{-0.25cm}
\begin{equation}
 \mathbb{E}\tau_{s}=\mathbb{E}\tau_{s-1}+\frac{U}{\mu(s)}\sum_{s'= s}^{\bar s}\prod_{i=s+1}^{s'}\frac{\lambda}{\mu(i)}
 \label{eq:hit}
 \vspace{-0.2cm}
\end{equation}
For $s<\bar s$, assume \eqref{eq:hit} is true for $\mathbb{E}\tau_{s+1}$:
\begin{align*}
 \mathbb{E}\tau_{s}&= 1+ \mathbb{E}\tau_{s+1}\frac{\lambda}{U}+\mathbb{E}\tau_{s}\frac{1-\mu(s)-\lambda}{U}+\mathbb{E}\tau_{s-1}\frac{\mu(s)}{U}\\
 &= 1+\mathbb{E}\tau_{s}\frac{1-\mu(s)-\lambda}{U}+\mathbb{E}\tau_{s-1}\frac{\mu(s)}{U}+\mathbb{E}\tau_s\frac{\lambda}{U}+\sum_{s'= s+1}^{\bar s}\prod_{i=s+1}^{s'}\frac{\lambda}{\mu(i)}\\
 &=\frac{U}{\mu(s)}+\mathbb{E}\tau_{s-1}+\frac{U}{\mu(s)}\sum_{s'= s+1}^{\bar s}\prod_{i=s+1}^{s'}\frac{\lambda}{\mu(i)} \text{ by gathering the $\tau_s$ terms }\\
 &=\mathbb{E}\tau_{s-1}+\frac{U}{\mu(s)}\sum_{s'= s}^{\bar s}\prod_{i=s+1}^{s'}\frac{\lambda}{\mu(i)},
\end{align*}
the induction is therefore true, and we have: $\mathbb{E}\tau_s\leq \mathbb{E}\tau_{s-1}+\frac{U}{\mu(s)}\sm^{\pi}(0)^{-1}$.
\end{proof}

\begin{Pro}
 \label{pro:bias}
 For any policy $\pi$, define for $s \in \croc{1,\ldots, S}$ the variation of the bias
 \begin{equation*}
 \Delta^\pi(s):=h^{\pi}(s)-h^{\pi}(s-1)=\sum_{t=1}^{\infty} \prt{P^t(s,\cdot)-P^t(s-1,\cdot)}\mathbf{r}.
 \end{equation*}
 Remind that $\delta_{\max}:=\max_{s,a,a'} |r(s,a)-r(s-1,a')|=\frac{\lambda \rcost+\hcost}{U}$:
 $$\Delta^\pi(s)\leq\Delta(s):= 2\delta_{\max} \sm^{\pi^{\max}}(0)^{-1} \sum_{i=1}^s \frac{U}{\mu(i)}
 .$$
 Using the monotonicity of the rates $\mu$ from Lemma~\ref{lem:EqRates}, we therefore have :
 $$\Delta(s)\leq 2(\lambda \rcost + \hcost) \sm^{\pi^{\max}}(0)^{-1} s \frac{1}{\mu(1)}$$
\end{Pro}

\begin{proof}
 We will use an optimal coupling, that is, a coupling such that the following infimum is reached, as defined in \cite{Peres-2008}.
\begin{equation}
 \left\|P^t(s,\cdot)\hspace{-0.05cm}-\hspace{-0.05cm}P^t(s\hspace{-0.05cm}-\hspace{-0.05cm}1,\cdot) \right\|_{\rm TV}=\hspace{-0.05cm}\inf\{\mathbb{P}(X_t\neq Y_t)\hspace{-0.05cm}:\hspace{-0.05cm}(X_t,Y_t) \text{ couples $P^t(s,\cdot)$ and $P^t(s\hspace{-0.05cm}-\hspace{-0.05cm}1,\cdot)$}\}.
 \label{eq:coupling1}
\end{equation}
More precisely, let $X$ and $Y$ be Markov chains with transition matrix $P$ and starting states $X_1=s$, $Y_1=s-1$, coupled in the following way:
For each time-step $t\geq2$, let $U_t \sim \mathcal{U}([0,1])$ be a sequence of independent random variables sampled uniformly on $[0,1]$. We have:
\begin{equation}
X_{t+1}=
\begin{cases}
 X_t-1 & \text{if } U_t \leq \mu(X_t)\\
 X_t & \text{if } \mu(X_t)\leq U_t \leq 1-\lambda\\
 X_t+1 &\text{if } 1-\lambda \leq U_t
\end{cases}
\end{equation}
and define $Y_{t+1}$ the same way from $Y_t$. This coupling is optimal, but in particular we have:
$$(P^t(s,\cdot)-P^t(s-1,\cdot))\mathbf{r} \leq 2\mathbb{P}(X_t \neq Y_t)\delta_{\max}. $$
 We remind that $\tau_{s}$ is the random time needed for the Markov chain $X_t$ to hit $0$. The coupling time is lower than $\tau_{s}$:
 \begin{equation*}
  \mathbb{P}(X_t \neq Y_t) \leq \Pb\prt{\tau_{s\to 0} >t},
 \end{equation*}
so that summing over $t$ gives:
\begin{equation*}
 \Delta^\pi(s) \leq 2\delta_{\max} \mathbb{E}\tau_{s},
\end{equation*}
and using Lemma~\ref{Lem:hit0} and Lemma~\ref{lem:measure}:
\begin{equation*}
 \Delta^\pi(s)\leq 2\delta_{\max} \sm^{\pi^{\max}}(0)^{-1} \sum_{i=1}^s \frac{U}{\mu(i)}.
\end{equation*}
\end{proof}

\end{appendices}



\end{document}